\documentclass[runningheads]{llncs}

\usepackage{eccv}

\usepackage{eccvabbrv}

\usepackage{graphicx}
\usepackage{booktabs}

\usepackage{multirow}
\usepackage[vlined, ruled]{algorithm2e}
\usepackage[noend]{algpseudocode}

\usepackage[accsupp]{axessibility}  % Improves PDF readability for those with disabilities.

\newtheorem{prop}{Proposition}

\usepackage{hyperref}

\usepackage{orcidlink}

\begin{document}

% ---------------------------------------------------------------
% TODO REVIEW: Replace with your title
\title{Teddy: Efficient Large-Scale Dataset Distillation via Taylor-Approximated Matching} 

% TODO REVIEW: If the paper title is too long for the running head, you can set
% an abbreviated paper title here. If not, comment out.
\titlerunning{Teddy: Efficient Large-Scale Dataset Distillation}

% TODO FINAL: Replace with your author list. 
% Include the authors' OCRID for the camera-ready version, if at all possible.
\author{Ruonan Yu\inst{}\orcidlink{0009-0008-4809-7119} \and
Songhua Liu\inst{}\orcidlink{0000-0003-1033-5122} \and
Jingwen Ye\inst{}\orcidlink{0000-0001-8415-3597} \and
Xinchao Wang\inst{}\orcidlink{0000-0003-0057-1404}\thanks{Corresponding author.}}

% TODO FINAL: Replace with an abbreviated list of authors.
\authorrunning{Yu et al.}
% First names are abbreviated in the running head.
% If there are more than two authors, 'et al.' is used.

% TODO FINAL: Replace with your institution list.
\institute{National University of Singapore \\
\email{\{ruonan,songhua.liu\}@u.nus.edu}, \email{\{jingweny,xinchao\}@nus.edu.sg}} 

% \institute{Princeton University, Princeton NJ 08544, USA \and
% Springer Heidelberg, Tiergartenstr.~17, 69121 Heidelberg, Germany
% \email{lncs@springer.com}\\
% \url{http://www.springer.com/gp/computer-science/lncs} \and
% ABC Institute, Rupert-Karls-University Heidelberg, Heidelberg, Germany\\
% \email{\{abc,lncs\}@uni-heidelberg.de}}

\maketitle

\begin{abstract}
Dataset distillation or condensation refers to compressing a large-scale dataset into a much smaller one, enabling models trained on this synthetic dataset to generalize effectively on real data. Tackling this challenge, as defined, relies on a bi-level optimization algorithm: a novel model is trained in each iteration within a nested loop, with gradients propagated through an unrolled computation graph. However, this approach incurs high memory and time complexity, posing difficulties in scaling up to large datasets such as ImageNet. 
Addressing these concerns, this paper introduces \textbf{Teddy}, a \textbf{T}aylor-approximat\textbf{e}d \textbf{d}ataset \textbf{d}istillation framework designed to handle large-scale dataset and enhance efficienc\textbf{y}. 
On the one hand, backed up by theoretical analysis, we propose a memory-efficient approximation derived from Taylor expansion, which transforms the original form dependent on multi-step gradients to a \textbf{first-order} one. On the other hand, rather than repeatedly training a novel model in each iteration, we unveil that employing a pre-cached pool of \textbf{weak} models, which can be generated from a \textbf{single} base model, enhances both time efficiency and performance concurrently, particularly when dealing with large-scale datasets. Extensive experiments demonstrate that the proposed Teddy attains state-of-the-art efficiency and performance on the Tiny-ImageNet and original-sized ImageNet-1K dataset, notably surpassing prior methods by up to 12.8\%, while reducing 46.6\% runtime. Our code will be available at \href{https://github.com/Lexie-YU/Teddy}{https://github.com/Lexie-YU/Teddy}.
  % \keywords{Dataset distillation \and Efficient training}
\end{abstract}

\section{Introduction}
\label{sec:intro}
Deep learning has made remarkable strides in various fields, largely owing to the utilization of extensive training data~\cite{deng2009imagenet,dosovitskiy2020image,krizhevsky2012imagenet,simonyan2014very,ye2024mutual}. However, there are associated expenses; achieving optimal performance necessitates substantial computational resources and a significant amount of GPU time, particularly for the latest state-of-the-art models~\cite{ramesh2022hierarchical,dosovitskiy2020image,bubeck2023sparks,rombach2022high}. Concerning these issues, a promising and emerging field, dataset distillation~(DD), also known as dataset condensation~(DC), has recently garnered significant attention. DD concentrates on condensing the extensive original dataset into a much smaller version while retaining the entirety of the dataset knowledge so that models trained on the synthetic dataset can generalize well on real data. 
\begin{figure}[t]
  \centering
   \includegraphics[width=\linewidth]{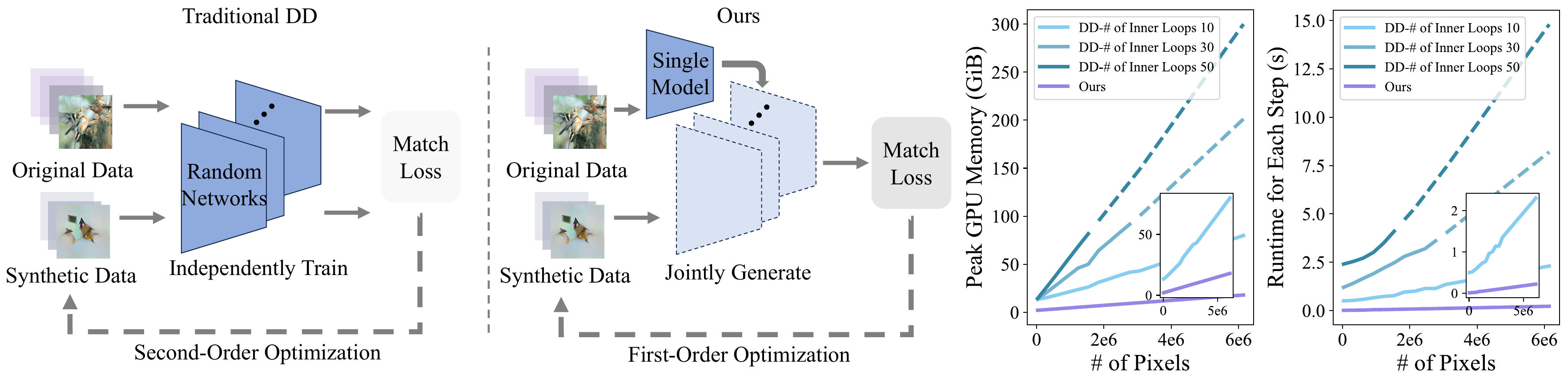}
   \caption{Illustration of meta-learning-based methods, our method~(left), and the comparison of memory and time efficiency~(right). Our proposed method exhibits surprising memory and time efficiency.}
   \label{fig:intro}
\end{figure}

Following the definition of DD, Wang \textit{et al.}~\cite{wang2018dataset} in seminal propose a principled meta-learning-based algorithm: in each iteration, a novel model is trained on the synthetic dataset in a nested loop, and the trained model is meta-tested on original data, and the gradient is backpropagated through an unrolled computation graph to update current synthetic data. 
Although effective, such bi-level optimization results in high memory and time complexity in practice. 
Specifically, on the one hand, each update of the current synthetic dataset necessitates an iterative meta-train process from a new model, making the time required for each step notably prolonged. 
On the other hand, caching the whole computation graph for back-propagation demands significant GPU memory. 
As illustrated in Fig.~\ref{fig:intro}, for instance, under the setting of a 10-class dataset with 50 images per class, each update requires over 50 GiB GPU memory, and the entire data generation takes hundreds of hours, let alone when condensing the large-scale ImageNet dataset, which has thousands of times optimization pixel space of the small dataset like CIFAR10.

To address the challenges above, subsequent research endeavors have focused on devising surrogate objectives to mitigate the bi-level nature of DD. While experiments have demonstrated their enhanced efficiency and effectiveness, these approaches still necessitate extensive forward and backward passes in neural networks for computation. 
For instance, methods such as gradient matching~\cite{zhao2020dataset,zhao2021datasetdsa} entail sampling gradients of current synthetic datasets and original data across a multitude of networks. 
Similarly, distribution matching~\cite{zhao2023datasetd,wang2022cafe} involves computing feature statistics within numerous networks. 
Additionally, training-trajectory-matching techniques~\cite{cazenavette2022dataset,cui2023scaling} necessitate the preparation of hundreds of training trajectories containing thousands of model checkpoints altogether. 
The substantial computational demands of current algorithms pose significant challenges when scaling up to handle large-scale datasets. 

In light of the analysis presented above, we rethink and deeply explore the primary DD optimization objectives in this paper. 
Concentrating on the main bi-level optimization problem, we propose an efficient approximation solution derived from Taylor expansion to the original one, transforming the original paradigm dependent on multi-step gradients or the second-order optimization to a first-order one. Moreover, we further simplify this Taylor approximation solution by pre-caching a pool of weak models, which enhances time efficiency and has also been verified as a more advantageous alternative to training models from scratch or using fully converged ones. 
Specifically, two practical cases are taken into consideration here to obtain weak teachers: prior and post-generation. Both efficiently generate the model pool from a single base model. 

To demonstrate the rationality of our proposed method, we provide a comprehensive theoretical analysis. This analysis also reveals the relationships between the mainstream optimization objectives adopted by existing methods and builds a unified theoretical framework of DD. 
Extensive experiments also validate the superiority of the proposed Teddy. It outperforms previous state-of-the-art methods by a large margin, especially for the original-sized ImageNet-1K dataset. 
Our contributions can be summarized below:
\begin{itemize}
    \item We propose an efficient large-scale dataset distillation method termed Teddy, a Taylor-approximated version of the solution to the DD. Our proposed method decouples the bi-level optimization and avoids model training within inner loops for each iteration, which can better handle large-scale datasets.
    \item We provide a comprehensive theoretical analysis for our proposed method, and we are the first to give a theoretical analysis linking the existing mainstream dataset distillation objectives used in drastically different dataset distillation methods. Our study establishes a unified theoretical framework for dataset distillation.
    \item We pre-cache weak models to decrease the time complexity introduced by training student models in nested loops for each iteration. Specifically, we consider two practical cases, the prior and post model pool generation, both efficiently generating the model pool from a single base model.
    \item We conduct comprehensive experiments on our proposed method, which show state-of-the-art performance. Notably, on the full-sized ImageNet-1K dataset, our proposed method outperforms previous methods by a substantial margin, reaching up to 12.8\%.
\end{itemize}

\section{Related Works}
\label{sec:related_works}
Dataset distillation~(DD), first proposed by Wang \textit{et al.}~\cite{wang2018dataset}, aims to condense the large-scale dataset into a small synthetic dataset and, when performed as the training data, preserve the performance of the trained models. DD addresses the tremendous computational cost issues and improves the efficiency of downstream tasks.
In recent years of rapid development, increasing research has emerged in the field of DD. Based on the variations in the optimization objectives, it can be broadly categorized into three classes: performance matching~\cite{wang2018dataset,nguyen2020dataset,zhou2022dataset,liu2024mgdd}, parameter matching~\cite{zhao2020dataset,zhao2021datasetdsa,cazenavette2022dataset,liu2023slimmable,du2023minimizing,cui2023scaling,ye2024distilled}, and distribution matching~\cite{zhao2023datasetd,wang2022cafe,zhao2023improved}. Moreover, some research is dedicated to improving the efficiency of data storage space utilization and expanding the learnable information space, such as label distillation~\cite{bohdal2020flexible,sucholutsky2021soft}, dataset parameterization~\cite{kim2022dataset,deng2022remember,liu2022dataset,lee2022dataset,zhao2021datasetdsa}, and model augmentation~\cite{zhang2023accelerating}, further enhancing the distilled performance. 

However, previous works in DD still have problems scaling up, including generating larger synthetic datasets, adapting to more extensive training networks, and handling more complex and massive original datasets~\cite{yu2023dataset,cui2022dc}. 
% One significant reason for this issue is the bilevel optimization process, especially for performance matching and parameter matching, in which each update for synthetic dataset requires backpropagation through the unrolled graph, as the training model is updated using the synthetic dataset within the inner loop, is high-memory demanding. 
One significant reason for this issue is the bilevel optimization process, where each synthetic dataset update requires backpropagation through the unrolled graph, making it high-memory demanding.
For this issue, one intuitive idea is to decouple the bilevel optimization process and unbind the network and the synthetic dataset~\cite{yu2023dataset}, such as the practice of DM~\cite{zhao2023datasetd} and IDM~\cite{zhao2023improved}, treating the network as the feature extractor without updating it or updating using the original dataset, and matching the feature distribution of the original and the synthetic datasets. However, there continues to be a substantial performance disparity with the performance and parameter matching methods, and recalculating the mean of the original dataset features for every iteration is an additional overhead. 

A recent work, proposed by Yin \textit{et al.}~\cite{yin2023squeeze}, utilizes the statistic information running mean and running variance of the original training datasets stored in the batch normalization layers of the well-trained model. They only adopt one single model, which may lack a comprehensive view of the entire dataset and lead to the mode collapse problem. Moreover, statistics information stored in batch normalization layers of well-trained models may overemphasize long-tail samples, which is challenging for the synthetic dataset to fit and be optimized.

\section{Methodology}
This section provides a comprehensive introduction to the proposed method, Teddy. We will begin by elucidating the underlying motivation, followed by an exhaustive presentation of our algorithm and theoretical analysis.

\subsection{Preliminary}
Given the original dataset, denoted as $\mathcal{T}=(X_t, Y_t)$, where $X_t \in \mathbb{R}^{N_t \times d}$ and $Y_t \in \mathbb{R}^{N_t \times c}$,
the primary goal of DD is to generate a synthetic dataset $\mathcal{S}=(X_s, Y_s)$, where $X_s \in \mathbb{R}^{N_s \times d}$ and $Y_s \in \mathbb{R}^{N_s \times c}$, such that the model training on it can obtain a comparable performance with the original dataset. Here, $N_t$ and $N_s$ are the number of samples in $\mathcal{T}$ and $\mathcal{S}$, respectively, and $N_s \ll N_t$. $d$ is the number of features for each sample. $c$ is the number of classes for classification tasks.
The objective can be formulated as:
\begin{equation}
\begin{split}
    \mathcal{S} &= \mathop{\arg\min}\limits_{\mathcal{S}} \mathbb{E}_{\theta^{(0)} \sim \Theta}[l(\mathcal{T};\theta^{(T)}_{\mathcal{S}})], \\
    \theta^{(t)}_{\mathcal{S}} &= \theta^{(t-1)}_{\mathcal{S}}-\alpha\nabla_{\theta^{(t-1)}_{\mathcal{S}}} l(\mathcal{S}; \theta^{(t-1)}_{\mathcal{S}}), 
\end{split}
\label{eq:meta}
\end{equation}
where $\Theta$ is the distribution of model initialization, $T$ is the number of epochs the model trained on the synthetic dataset,  $1\leq t\leq T$, $\alpha$ is the learning rate, and $l(\cdot; \theta)$ is the cross-entropy loss for classification tasks. 

As Eq.~\ref{eq:meta} shows, the objective involves a bi-level optimization process. Each step of updating the synthetic dataset needs to backpropagate through an unrolled computation graph, making it inefficient for both memory and time. Moreover, it should train a novel model in a nested loop for every iteration, resulting in notable time consumption, particularly when handling complex datasets. 

\begin{figure*}[t]
  \centering
  \begin{subfigure}{0.3\linewidth}
    \includegraphics[width=\linewidth]{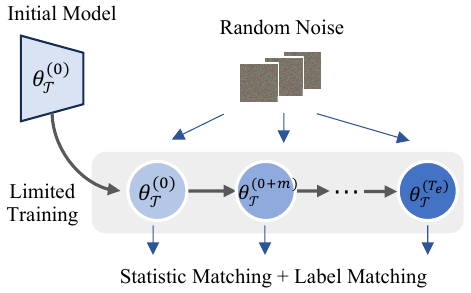}
    \caption{Prior Generation}
    \label{fig:method_prior}
  \end{subfigure}
  \begin{subfigure}{0.3\linewidth}
    \includegraphics[width=\linewidth]{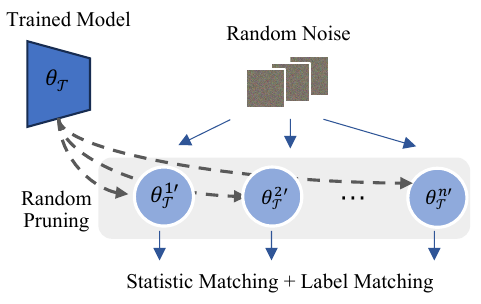}
    \caption{Post Generation}
    \label{fig:method_post}
\end{subfigure}
  \hfill
  \begin{subfigure}{0.37\linewidth}
    \includegraphics[width=0.97\linewidth]{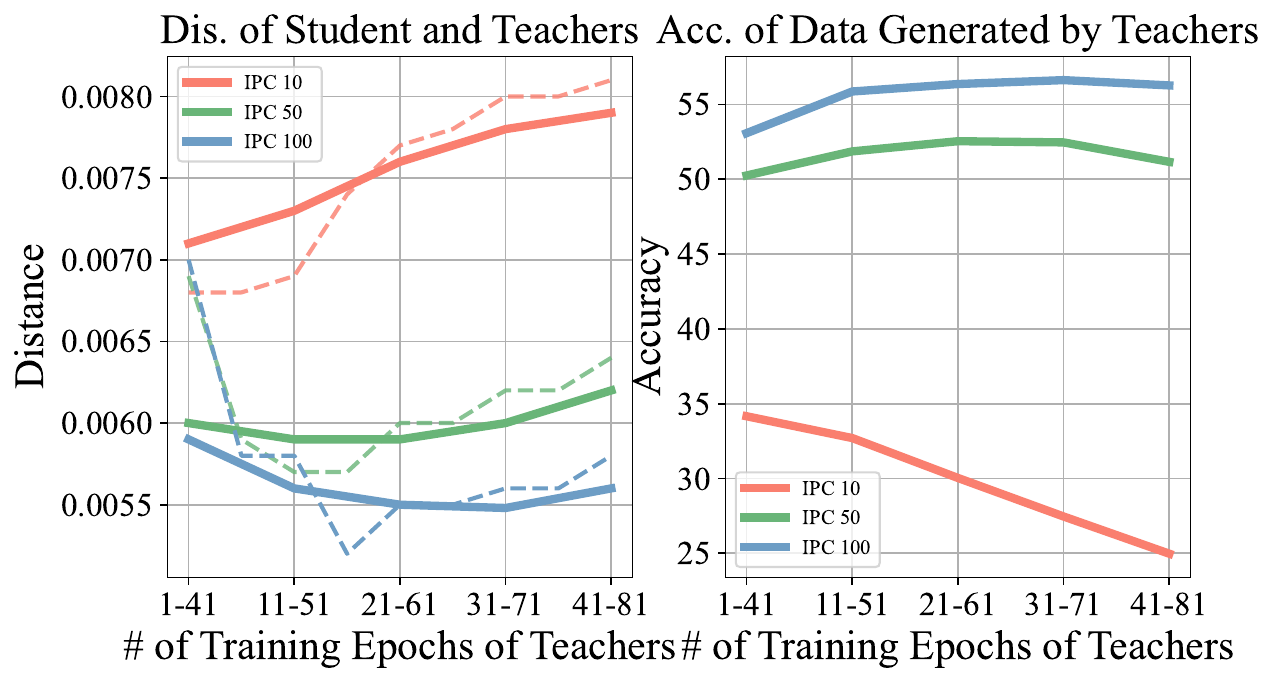}
    \caption{Distance v.s. Accuracy}
    \label{fig:motivate}
  \end{subfigure}
  \caption{(a)(b) Illustration of our proposed method. Firstly, we generate the model pool by prior or post-generation. Then, update the initialized distilled data via statistic and label matching. Lastly, input the generated distilled data with augmentation into the model pool again to obtain the soft label. (c) The left figure shows the distance between student and teacher models at different stages. The dotted line represents the distance between a single teacher and a student. The solid line represents the average of the distances within the range. The right figure shows the performance of distilled data generated from teacher models at different stages.}
\end{figure*}

\subsection{Taylor-Approximated Matching}
In this section, we start with the fundamental solution to DD, the meta-learning-based method, progressively delving into our proposed method, Teddy, with theoretical analysis.
The DD formulation is shown in Eq.~\ref{eq:meta}. Practically, to enhance the generalization ability of the distilled data on the real data, the inner loop is typically configured with multiple steps, \textit{e.g.}, $T>10$. 
Considering the synthetic dataset is of small size, the model is highly susceptible to overfitting on it as training for $T$ epochs. 
Therefore, it has $l(\mathcal{S};\theta^{(T)}_{\mathcal{S}})<\epsilon$, where $\epsilon$ is close to zero. 
Employing the Karush-Kuhn-Tucker~(KKT) conditions, the optimization objective in Eq.~\ref{eq:meta} can be transformed into the following: 
\begin{equation}
\begin{split}
    \mathcal{S} &= \mathop{\arg\min}\limits_{\mathcal{S}}[l(\mathcal{T};\theta^{(T)}_{\mathcal{S}}) +u( l(\mathcal{S};\theta^{(T)}_{\mathcal{S}})-\epsilon)], \\
    \theta^{(t)}_{\mathcal{S}} &= \theta^{(0)} -\alpha\sum_{i=0}^{t-1}g_{\mathcal{S}}^{(i)} = \theta^{(t-1)}_{\mathcal{S}}-\alpha g^{(t-1)}_{\mathcal{S}},
\end{split}
\label{eq:meta2}
\end{equation}
where $\theta^{(0)}\sim\Theta$, $g_{\mathcal{S}}^{(t)}=\nabla_{\theta^{(t)}_{\mathcal{S}}} l(\mathcal{S}; \theta^{(t)}_{\mathcal{S}})$, and $u$ is the Lagrange multiplier. 
\begin{prop}
The meta-learning-based optimization objective can be Taylor-approximated as the sum of the gradient matching of the distilled data and the original data for all steps along the training trajectory of the student model.
\label{prop:taylor}
\end{prop}
\begin{proof}
Here, we recursively apply the first-order Taylor expansion to unfold the unrolled computational graph. The first term of the optimization objective can be transformed into follows:
\begin{equation}
\begin{split}
     l(\mathcal{T};\theta^{(T)}_{\mathcal{S}}) 
     &= l(\mathcal{T};\theta^{(T-1)}_{\mathcal{S}}-\alpha g_{\mathcal{S}}^{(T-1)}) \approx l(\mathcal{T};\theta^{(T-1)}_{\mathcal{S}})-\alpha g_{\mathcal{T}}^{(T-1)}\cdot g_{\mathcal{S}}^{(T-1)} \\
     &\approx l(\mathcal{T};\theta^{(T-i)}_{\mathcal{S}})-\alpha \sum_{t=T-i}^{T-1}g_{\mathcal{T}}^{(t)}\cdot g_{\mathcal{S}}^{(t)} \approx l(\mathcal{T};\theta^{(0)})-\alpha \sum_{t=0}^{T-1}g_{\mathcal{T}}^{(t)}\cdot g_{\mathcal{S}}^{(t)},
\end{split}
\label{eq:taylor}
\end{equation} 
where $\theta^{(0)}\sim\Theta$, and $g_{\mathcal{T}}^{(t)}=\nabla_{\theta^{(t)}_{\mathcal{S}}} l(\mathcal{T}; \theta^{(t)}_{\mathcal{S}})$. The first term $l(\mathcal{T};\theta^{(0)})$ is the loss of the original dataset on the initial model $\theta^{(0)}$, and irrelevant to the optimization problem \textit{w.r.t} $\mathcal{S}$. Concluded from Eq.~\ref{eq:taylor}, the meta-learning-based optimization objective can be transformed into the sum of the gradient matching for all steps pathing the student model training trajectory. 
\end{proof}
However, updating the distilled data should compute and accumulate the gradient matching loss at each step within the inner loop. It involves a second-order optimization problem, which may cause high computational complexity.
To tackle this issue, we transfer the optimization objective into a first order one.
\begin{prop}
Gradient matching is equivalent to first-order and second-order statistic information matching in feature space if the dataset classes are balanced.
\label{prop:gradtodis}
\end{prop}
\begin{proof}
For simplicity, we assume the case in FRePo~\cite{zhou2022dataset}, where only the last linear layer of the network~($\theta$) is updated for synthetic data, denoted as $W$, while the preceding layers~($f_{\theta}$) serve as the feature extractor. In this case, the gradient can form as follows:
\begin{equation}
\begin{split}
    g &= \frac{\partial l(\cdot;\theta)}{\partial W} =\frac{1}{|X|} f_{\theta}(X)^T(f_{\theta}(X)W-Y)\\
      &= \frac{1}{|X|}f_{\theta}(X)^Tf_{\theta}(X)W - \frac{1}{|X|}f_{\theta}(X)^TY \\
      &\overset{\theta^{(0)} \sim \Theta}{=} \frac{1}{|X|}\sigma^2(f_{\theta}(X))W-\frac{1}{|X|}\mu(f_{\theta}(X)).
\end{split}
\label{eq:gradtodis}
\end{equation}
The first term of Eq.~\ref{eq:gradtodis} is the weighted correlation of the extracted features of the dataset, and the second term is the class-wise mean of the extracted features. Under the condition of $\theta^{(0)}\sim\Theta$, the constraining covariance is essentially equivalent with variance. Also, when classes in the dataset are balanced, the class-wise mean can be replaced by global mean. For detailed proofs, please refer to the supplementary materials.
\end{proof}
Therefore, the optimization objective in Eq.~\ref{eq:meta2} can be transferred to:
\begin{equation}
\begin{split}
    l(\mathcal{T};\theta^{(T)}_{\mathcal{S}})&\approx\mathbb{E}_{\theta^{(0)}\sim\Theta}[\sum_{t=0}^{T-1}(\sum_l||\mu_l(f_{\theta_{\mathcal{S}}^{(t)}}(X_s))-\mu_l(f_{\theta_{\mathcal{S}}^{(t)}}(X_t))||_2 \\
     &+\sum_l ||\sigma_l^2(f_{\theta_{\mathcal{S}}^{(t)}}(X_s))-\sigma_l^2(f_{\theta_{\mathcal{S}}^{(t)}}(X_t))||_2) 
     +u\cdot l(\mathcal{S};\theta^{(T)}_{\mathcal{S}})],
\end{split}
\label{eq:dis}
\end{equation}
where $\mu_l$ and $\sigma^2_l$ refer to the mean and variance of the $l^{th}$ layer features. 
Here, we solve the core issue of bi-level optimization by stripping the inner loop training process from the backpropagation computational graph. That is, we match the statistical information of the original and synthetic dataset in the feature space. 
Specifically, we train a single student trajectory for every iteration from scratch. For each checkpoint along the trajectory, we perform the statistic information~(the mean and variance in each layer) matching of the original
and synthetic datasets, which largely improves the efficiency.

\subsection{Model Pool Generation}
Considering that we have already improved the efficiency by Eq.~\ref{eq:dis}, it is still time demanding since we need to re-train the new model for each synthetic data updating.
It will become significantly pronounced when handling the large-scale datasets. 

To solve this problem and improve the efficiency, we should deal with the long student training trajectory and retraining student model for each iteration. We notice that for each synthetic data updating iteration, any segments of the student model training trajectory that begin with $t$ and end with $t+m$ have: 
\begin{equation}
l(\mathcal{T};\theta_{\mathcal{S}}^{(t+m)})=l(\mathcal{T};\theta_{\mathcal{S}}^{(t)}-\alpha\sum_{i=t}^{t+m-1}g_{\mathcal{S}}^{(i)})\approx l(\mathcal{T};\theta^{(t)}_{\mathcal{S}})-\alpha (\sum_{i=t}^{t+m-1}g_{\mathcal{S}}^{(i)})\cdot g_{\mathcal{T}}^{(t)},
\label{eq:multi-step}
\end{equation}
of which the $(\sum_{i=t}^{t+m-1}g_{\mathcal{S}}^{(i)})\cdot g_{\mathcal{T}}^{(t)}$ shows that multi-step student model updating $\sum_{i=t}^{t+m-1}g_{\mathcal{S}}^{(i)}$ is comparable with single-step teacher updating $g_{\mathcal{T}}^{(t)}$ for appropriate interval $m$. Thus, each student trajectory segmentation can be replaced by a single \textbf{comparable} teacher model. The ``comparable'' here means the teachers have the same-level performance as students~(trained on the synthetic dataset), which does not reach convergence on the real data and is \textbf{weak} in performance. Also, the distance between the students and teachers is close. 

Insight from Eq.~\ref{eq:multi-step}, to further improve the time efficiency, 
we construct the weak teacher model pool from two practical cases, the prior and post-generation. These strategies are both efficiently generating model pool from a single base model, as shown in Fig.~\ref{fig:method_prior} and Fig.~\ref{fig:method_post}. 
Here, as we utilize weak teachers instead of fully-convergence ones, we can just use a single model as the base to generate diverse weak teachers. This strategy does not compromise performance.
More specifically, if the base model is:
\begin{itemize}
    \item at the early stage or randomly initialized, we adopt the prior-generation strategy. We pre-cache the weak teachers from the early stage of the training trajectory for every $m$ steps, substituting the whole student training trajectory.
    \item at the late stage or well-trained, we propose an alternative pruning-based post-generation model pool method. Here, we employ the DepGraph~\cite{fang2023depgraph} with random strategy, complemented by fine-tuning with very limited~(\textit{e.g.}, 0-2) steps. 
\end{itemize}
Additionally, when dealing with extensive models and datasets, training weak teachers, even with a limited number of epochs, still remains a significant overhead. In this case, post-generation allows for the rapid generation of many weak models with low computational costs. Also, pruned models exhibit enhanced inference speeds, further boosting the efficiency of our method. 

To verify the rationality of this strategy, we conduct validation experiments. We measure the distance between the student and teacher models at different stages and the accuracy of the distilled data generated from teacher models at different stages. Here, we use KL divergence to measure model distance, capturing the information loss when approximating one probability distribution with another. The results shown in Fig.~\ref{fig:motivate} reveal a strong inverse correlation between model distance and distilled data performance. Closer distances indicate better approximation and superior performance.

Here, as we replace the student models with weak teachers, we can utilize the running mean and running variance for the original dataset from the batch normalization layers in the weak teachers instead of recalculating multiple times for each iteration. It will reduce a lot of computational costs, especially for large-scale datasets.
Additionally, the classifier prior $l(\mathcal{S};\theta^{(T)}_{\mathcal{S}})$ can also be replaced by ensemble classifier prior, as weak teachers may mislead generation direction. 
It will also benefit better alignment with the target classes. 
In summary, the final optimization objective is as follows:
\begin{equation}
\begin{split}
    &\sum_{t=T_b}^{T_e}(\sum_l||\mu_l(f_{\theta_{\mathcal{T}}^{(t)}}(X_s))-RM_{\theta^{(t)}_{\mathcal{T}}}^l(X_t))||_2 \\
     &+\sum_l ||\sigma_l^2(f_{\theta_{\mathcal{T}}^{(t)}}(X_s))-RV_{\theta^{(t)}_{\mathcal{T}}}^l(X_t))||_2 +u\cdot l(\mathcal{S};\theta^{(t)}_{\mathcal{T}})),
\end{split}
\label{eq:objective}
\end{equation}
where $RM$ and $RV$ refer to the running mean and variance stored in the batch normalization layers. $T_b$ and $T_e$~(both at the early stage) are the starting and end points of the teacher training trajectory for statistic information matching. We only consider performing the matching for every $m$ steps from $T_b$ to $T_e$.
For more theoretical details, please refer to the supplementary materials.

\begin{algorithm}[t]
\caption{Teddy Framework}
\KwIn{Original dataset $\mathcal{T}$, single base model $\theta_{base}$} 
\KwOut{Synthetic dataset $\mathcal{S}$}
\label{framework}
Initialize $\mathcal{S}$\\
\If{$\theta_{base}$ is from random or at early stage}{Prior-generate model pool $\mathcal{M}$}
\ElseIf{$\theta_{base}$ is well-trained or at late stage}{Post-generation model pool $\mathcal{M}$}

\While{not converge}{
Randomly select $n$ models from $\mathcal{M}$ \\
Compute $\mathcal{L}(\mathcal{S},\mathcal{T})$ as Eq.~\ref{eq:objective}\\
Back-propagate and update $\mathcal{S}$\\
}
Ensembly generate soft label via $\mathcal{M}$, $Y_s = \frac{1}{|\mathcal{M}|}\sum_{\theta \in \mathcal{M}} h(\mathcal{A}(X_s);\theta)$\\
\Comment{$h(\cdot;\theta)$ represents the model with parameter $\theta$, $\mathcal{A}$ is the function of data augmentation.}\\
\Return{$\mathcal{S}$}
\end{algorithm}
\subsection{Algorithm Summary}
In summary, we propose an efficient large-scale dataset distillation method, Teddy, solving the core bi-level optimization problem and further improving the efficiency via Taylor approximation. The framework of Teddy is shown in Algorithm \ref{framework}.
We first generate model pool $\mathcal{M}$ from single base model $\theta_{base}$. According to the state of the $\theta_{base}$, we choose from two efficient model pool generation strategies, prior and post-generation. Then, we follow our proposed optimization objective, Eq.~\ref{eq:objective}, to update the synthetic dataset, which only requires the first-order optimization process.
Here, considering the weak teacher models may mislead the data generation direction, we generate soft labels via weak teacher ensemble, which provides more richer label information expression. 

\section{Experiments}
In this section, we validate the efficacy of our proposed method, Teddy, via extensive experiments. 
Note that Teddy is designed for large-scale dataset distillation; we focus on two more challenging datasets, Tiny-ImageNet and full-size ImageNet.
We also compare Teddy on the small datasets like CIFAR10 with fair comparison, showing its comparative performance~(it is included in supplementary).
We evaluate the cross-architecture generalization performance of our generated synthetic dataset and perform comprehensive ablation studies to demonstrate the effectiveness of the strategies adopted in our method.

\subsection{Experiment Setting}
\subsubsection{Datasets and Networks}
To validate the efficacy and the efficiency of Teddy for large-scale datasets, we adopt two challenging datasets, 64 $\times$ 64 Tiny-ImageNet and 224 $\times$ 224 ImageNet-1K. While some previous methods have demonstrated effectiveness on downsized versions of ImageNet-1K or subsets, Teddy achieves impressive performance on the original, full-size ImageNet-1K.
As for Tiny-ImageNet, following the instruction of MoCo~(CIFAR)~\cite{he2020momentum} and the prior work SRe$^2$L~\cite{yin2023squeeze}, we adopt modified ResNet18 as the teacher model by replacing the first Conv layer with the 3$\times$3 kernel, stride 1 Conv layer, and removing the first maxpool. For ImageNet-1K, we use ResNet18~\cite{he2016deep}, implemented by the official torchvision.
Moreover, to show the generalization capabilities of our method, we evaluate the cross-architecture classification performance of ImageNet-1K with IPC 10 on several model architectures: ResNet50~\cite{he2016deep}, ResNet101~\cite{he2016deep}, DenseNet121~\cite{huang2017densely}, MobileNetV2~\cite{sandler2018mobilenetv2}, ShuffleNetV2~\cite{ma2018shufflenet}, and EfficientNetB0~\cite{tan2019efficientnet}.

\begin{table}[!t]
  \centering
  \resizebox{\linewidth}{!}{
  \begin{tabular}{c c c c c c c c}
    \toprule
    \multirow{2}{*}{\bf Method} & & \multicolumn{2}{c}{\bf Tiny-ImageNet} & &\multicolumn{3}{c}{\bf ImageNet-1K} \\
    \cmidrule{3-4} \cmidrule{6-8}
    & & 50 & 100 & & 10 & 50 & 100 \\
    \midrule
    \bf Random~(Conv) && 15.1 $\pm$ 0.3 & 24.3 $\pm$ 0.3 && 4.1 $\pm$ 0.1$^*$ & 16.2 $\pm$ 0.8$^*$ & 19.5 $\pm$ 0.5$^*$\\
    \bf Random~(ResNet18) && 18.2 $\pm$ 0.2 & 25.0 $\pm$ 0.2 && 6.8 $\pm$ 0.1 & 32.0 $\pm$ 0.2 & 45.7 $\pm$ 0.1\\
    \midrule
    \bf DC~\cite{zhao2020dataset} && 11.2 $\pm$ 0.3 & - && - & - & - \\
    \bf DSA~\cite{zhao2021datasetdsa} && 25.3 $\pm$ 0.2 & - && - & - & - \\
    \bf DM~\cite{zhao2023datasetd} && 24.1 $\pm$ 0.3 & 29.4 $\pm$ 0.2 && - & - & - \\
    \bf IDM~\cite{zhao2023improved} && 27.7 $\pm$ 0.3 & - && - & - & - \\
    \bf MTT~\cite{cazenavette2022dataset} && 28.2 $\pm$ 0.5 & 33.7 $\pm$ 0.6 && - & - & - \\
    \bf FTD~\cite{du2023minimizing} && 31.5 $\pm$ 0.3 & 34.5 $\pm$ 0.4 && - & - & - \\
    \bf TESLA~\cite{cui2023scaling} && 33.4 $\pm$ 0.5 & 34.7 $\pm$ 0.2 && 17.8 $\pm$ 1.3$^*$ & 27.9 $\pm$ 1.2$^*$ & 29.2 $\pm$ 1.0$^*$ \\
    \bf SRe$^2$L~\cite{yin2023squeeze} && 41.1 $\pm$ 0.4 & 49.7 $\pm$ 0.3 && 21.3 $\pm$ 0.6 & 46.8 $\pm$ 0.2 & 52.8 $\pm$ 0.3 \\
    \midrule
    \bf Ours~(post) && 44.5 $\pm$ 0.2~(\textcolor{blue}{+ 3.4}) & 51.4 $\pm$ 0.2~(\textcolor{blue}{+ 1.7}) && 32.7 $\pm$ 0.2~(\textcolor{blue}{+ 11.4}) & \textbf{52.5 $\pm$ 0.1~(\textcolor{blue}{+ 5.7})} & 56.2 $\pm$ 0.2~(\textcolor{blue}{+ 3.4})\\
    \bf Ours~(prior) && \textbf{45.2 $\pm$ 0.1~(\textcolor{blue}{+ 4.1})} & \textbf{52.0 $\pm$ 0.2~(\textcolor{blue}{+ 2.3})} && \textbf{34.1 $\pm$ 0.1~(\textcolor{blue}{+ 12.8})} & \textbf{52.5 $\pm$ 0.1~(\textcolor{blue}{+ 5.7})} & \textbf{56.5 $\pm$ 0.1~(\textcolor{blue}{+ 3.7})}\\
    \bottomrule
  \end{tabular}
  }

  \caption{Comparison with baseline methods. $^*$ indicates the evaluation results on downsampled ImageNet-1K dataset. Here, SRe$^2$L and our proposed methods adopt the ResNet18 as the training and evaluation model, other methods adopt ConvNet.}
  \label{tab:baseline}
\end{table}

\begin{table}[!t]
    \centering
    \resizebox{\linewidth}{!}{
    \begin{tabular}{c c c c c c c}
        \toprule
        \textbf{Method}  & \bf ResNet50 & \bf ResNet101 & \bf DenseNet121 &\bf MobileNetV2 & \bf ShuffleNetV2 & \bf EfficientNetB0\\
        \midrule
        \bf SRe$^2$L~\cite{yin2023squeeze} & 28.4 $\pm$ 0.1 & 30.9 $\pm$ 0.1 & 21.5 $\pm$ 0.5 & 10.2 $\pm$ 0.2 & 29.1 $\pm$ 0.1 & 16.1 $\pm$ 0.1 \\
        \midrule
        \bf Ours~(post) & 37.9 $\pm$ 0.1~(\textcolor{blue}{+ 9.5}) & 40.0 $\pm$ 0.1~(\textcolor{blue}{+ 9.1}) & 33.0 $\pm$ 0.1~(\textcolor{blue}{+ 11.5}) & 20.5 $\pm$ 0.1~(\textcolor{blue}{+ 10.3}) & \textbf{40.0 $\pm$ 0.3~(\textcolor{blue}{+ 10.9})} & 27.3 $\pm$ 0.2~(\textcolor{blue}{+ 11.2})\\
        \bf Ours~(prior) & \textbf{39.0 $\pm$ 0.1~(\textcolor{blue}{+ 10.6})} & \textbf{40.3 $\pm$ 0.1~(\textcolor{blue}{+ 9.4})} & \textbf{34.3 $\pm$ 0.1~(\textcolor{blue}{+ 12.8})} & \textbf{23.4 $\pm$ 0.3~(\textcolor{blue}{+ 13.2})} & 38.5 $\pm$ 0.1~(\textcolor{blue}{+ 9.4}) & \textbf{29.2 $\pm$ 0.1~(\textcolor{blue}{+ 13.1})} \\
        \bottomrule
    \end{tabular}
    }
    \caption{Evaluation results of cross-architecture generalization under the ImageNet-1K with IPC 10 setting. SRe$^2$L and our methods use ResNet18 as the training model.}
    \label{tab:crossarch}
\end{table}

\subsubsection{Implementation Details}
For prior model pool generation,
we cache 9 checkpoints in the early stage of single teacher training trajectory to build the model pool for generating distilled data of ImageNet-1K, and 8 for Tiny-ImageNet. 
During the data generation phase, a subset of 3 models will be randomly selected to generate the $i^{th}$ image for all classes, alleviating the mode collapse problem.
As for the pruning-based model pool, we perform pruning on the off-the-shelf pre-trained ResNet18 provided by PyTorch, with the Top-1 accuracy of 69.76\%, and finetune for very limited epochs, \textit{e.g.}, 0-2 epochs. The GFLOPs of the target pruned model is 1.2G, and the number of parameters is 7.72M. Here, we save 10 models for generating distilled data of ImageNet-1K, and 9 for Tiny-ImageNet. We randomly select 4 models from the model pool to generate the synthetic dataset for ImageNet-1K and 5 models for Tiny-ImageNet.
Moreover, after generation, the augmented synthetic data will input the model pool again to get soft labels. For more details, please refer to the supplementary materials.
\subsubsection{Evaluation and Baselines}
Following prior works~\cite{wang2018dataset,zhao2020dataset,zhao2023datasetd}, we evaluate the distilled datasets by training several randomly initialized models from scratch and report the mean and standard deviation of their accuracy on the corresponding real test set. We compare our proposed methods with series methods, including DC~\cite{zhao2020dataset}, DSA~\cite{zhao2023datasetd}, DM~\cite{zhao2023datasetd}, IDM~\cite{zhao2023improved}, MTT~\cite{cazenavette2022dataset}, FTD~\cite{du2023minimizing}, TESLA~\cite{cui2023scaling} and SRe$^2$L~\cite{yin2023squeeze}.
Here, we present the results derived from their original paper or DD benchmarks~\cite{cui2022dc} if available for selected baseline methods. However, as most existing methods are challenging to apply to the full-sized ImageNet-1K, we also report results on downsized ImageNet-1K. Additionally, the methods adopted to compare in the baseline all employ ConvNet, and SRe$^2$L~\cite{yin2023squeeze} and our method utilize ResNet18 as the training and evaluation model.

\subsection{Results on Baselines}
As demonstrated in Table~\ref{tab:baseline}, our proposed method, Teddy, significantly surpasses all the baselines for both the Tiny-ImageNet and ImageNet-1K datasets with all IPC settings. Furthermore, comparing the results across the two datasets reveals that our method exhibits a more pronounced improvement on complex datasets. 
In particular, in the case of ImageNet-1K with IPC 10, our evaluation results outperform the previous state-of-the-art methods by a large margin up to 12.8\%. Moreover, our method achieves the same performance with the setting of IPC 50~(52.5\%) as the previous SOTA method in the case of IPC 100~(52.8\%). 
Additionally, for DD, enhancing the performance of the larger synthetic dataset is particularly challenging. 
This is because the existing DD methods focus mainly on key dataset information, ignoring long-tail information. With increased IPCs, there is less room to enhance as compressing these details is hard.
However, as shown in Table~\ref{tab:baseline}, our method considerably improves the performance for large IPCs~(5.7\% and 3.7\% for IPC 50 and 100).
For the post-generation, although we applied pruning to the teacher model, resulting in a reduction of statistic information space in its batch normalization layers, the results in Table~\ref{tab:baseline} indicate that the impact is not substantial. We can still achieve comparable performance.

\subsection{Results on Cross-Architecture Generalization}
The generalization capacity is vital for practically applying dataset distillation. 
% Nonetheless, existing dataset distillation methods face performance degradation across different architectures~\cite{cui2023scaling,yu2023dataset}.
Here, we evaluate the generalization performance of the distilled dataset of the ImageNet-1K IPC 10 setting. In selecting model architecture, we adopt the ResNet50, ResNet101, DenseNet121, MobileNetV2, ShuffleNetV2, and EfficientNetB0. All results are reported in Table~\ref{tab:crossarch}. 

The results show that our proposed method, Teddy, surpasses all previous state-of-the-art methods by a significant margin.
When the evaluation models have larger and deeper architectures, there is a significant and proportional performance improvement, comparing 34.1\%, 39.0\%, and 40.3\% for ResNet18, 50, and 101, with enhancements of 12.8\%, 10.6\% and 9.4\%. 
For structurally different networks, our method demonstrates strong generalization capabilities. For instance, on DenseNet121, MobileNetV2, ShuffleNetV2, and EfficientNetB0, we achieved 34.3\%, 23.4\%, 40.0\%, and 29.2\%, respectively, surpassing the SOTA by 12.8\%, 13.2\%, 10.9\%, and 13.1\%. 
As shown in Table~\ref{tab:crossarch}, the evaluation results of the post-generation method indicate a strong generalization capability, even with the architecture heterogeneity introduced by network pruning, surpassing the previous SOTA methods.

\subsection{Ablation Study}
To substantiate the efficacy of our method, we perform ablation studies explicitly focusing on the effect of the model pool and the weak teachers. The outcomes underscore the effectiveness of the employed strategies, leading to a substantial enhancement in the performance of our proposed method.

\begin{figure*}[t]
  \centering
  \begin{subfigure}{0.18\linewidth}
    \includegraphics[width=\linewidth]{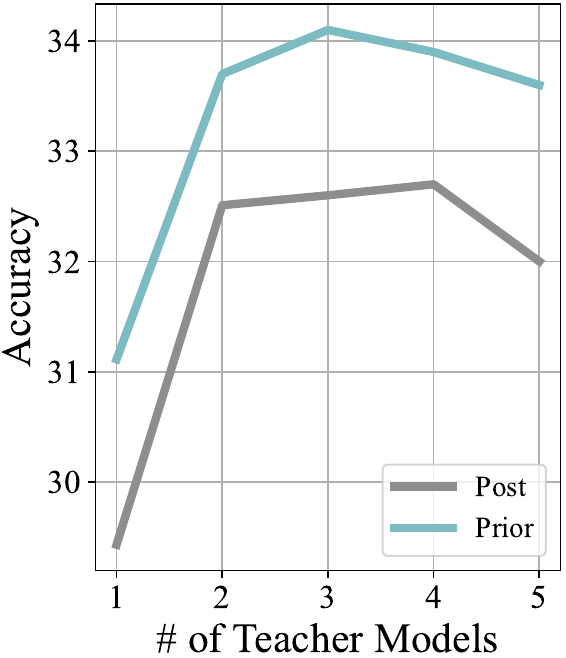}
    \caption{}
    \label{fig:ablation1}
  \end{subfigure}
  \begin{subfigure}{0.19\linewidth}
    \includegraphics[width=\linewidth]{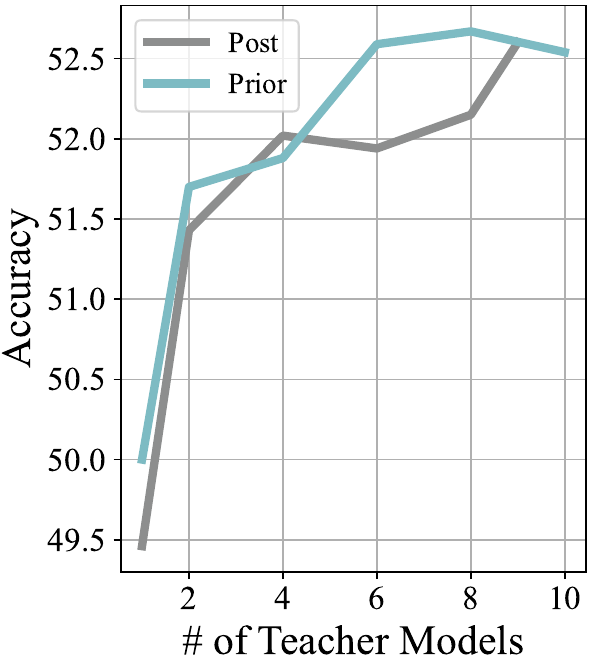}
    \caption{}
    \label{fig:ablation2}
\end{subfigure}
  \begin{subfigure}{0.18\linewidth}
    \includegraphics[width=\linewidth]{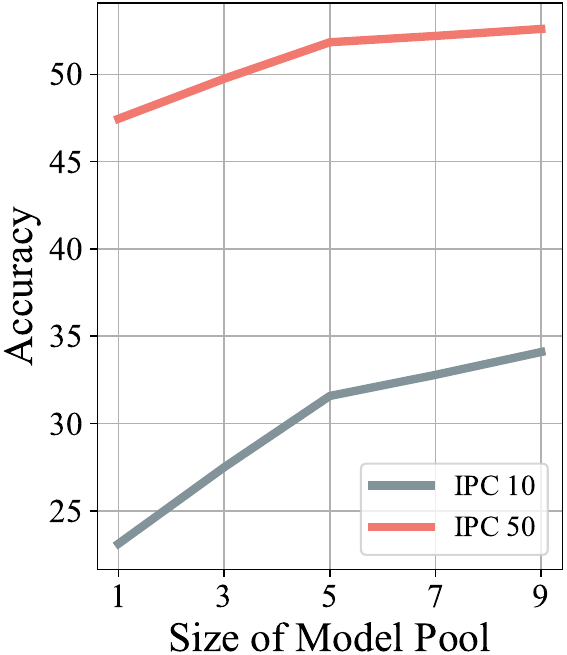}
    \caption{}
    \label{fig:poolsize}
  \end{subfigure}
  \begin{subfigure}{0.18\linewidth}
    \includegraphics[width=\linewidth]{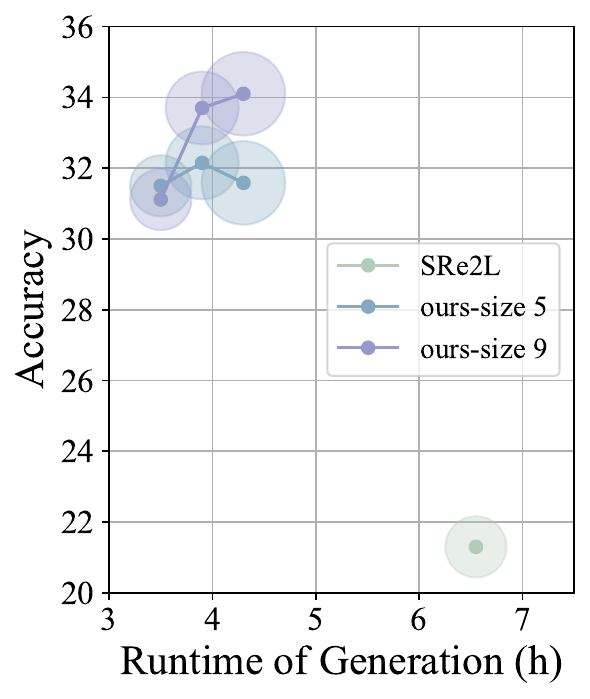}
    \caption{}
    \label{fig:efficiency}
  \end{subfigure}
  \begin{subfigure}{0.18\linewidth}
    \includegraphics[width=\linewidth]{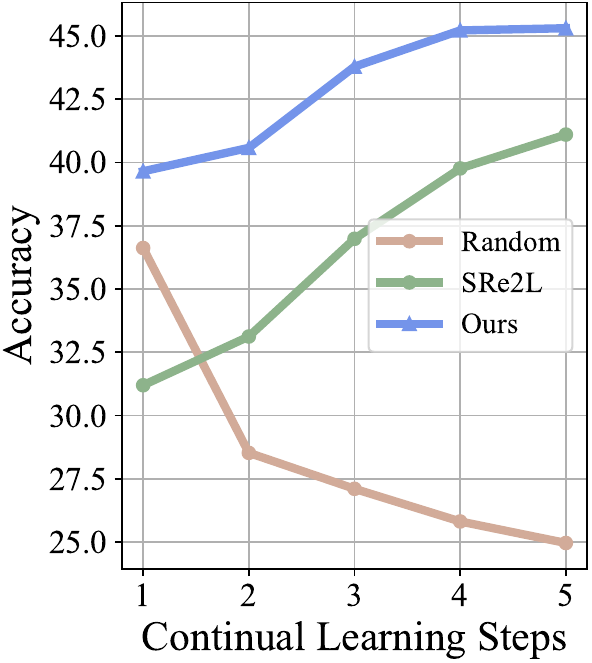}
    \caption{}
    \label{fig:continual}
\end{subfigure}
  \caption{(a) Ablation study on different number of models ensemble to generate synthetic data. (b) Ablation study on different number of models ensemble to generate soft label. (c) Ablation study on size of the model pool under the setting of ImageNet-1K IPC 10 and IPC 50 with the prior-generation strategy. (d) The time and memory requirement of our method compared with the previous SOTA. Here the size of the points represents the peak GPU memory, and the three points, from left to right, report the evaluation results of 1, 2, and 3 teacher models ensemble utilized in generating the synthetic data. (e) Continual learning on Tiny-ImageNet IPC 50 with 5-step incremental protocol.}
\end{figure*}

\subsubsection{Effect of Model Pool} 
We utilize the model pool of weak teachers as a substitute for students, reducing computational expenses. 
To analyze effectiveness, we conduct separate experiments for data generation and relabel phases.

For the data generation part, we conduct experiments on the setting of ImageNet-1K with IPC 10, and the setting of the model pool is the same as the main table.
For data generation, 1, 2, 3, 4, or 5 models will be randomly ensemble selected. 
The soft labels are generated by the whole model pool.
The results in Fig.~\ref{fig:ablation1} reveal that the performance is better with multiple models than with a single model. The accuracy drops in Fig.~\ref{fig:ablation2} can be attributed to the trade-off between trying fit numerous models comprehensively or concentrating on fitting fewer ones. Also, our method exhibits robustness in terms of the number of models selected for generation.

As for the relabel part, we conduct experiments on the setting of ImageNet-1K with IPC 50, and the setting of the model pool is the same as the main table.
For data generation, 3 models are randomly selected from the model pool. Those generated data will be ensemble relabeled by 1, 2, 4, 6, 8, or whole models of the model pool. The evaluation results are shown in Fig.~\ref{fig:ablation2}. We observe a significant enhancement in performance with ensemble relabeling, but also robust for the number of models selected for relabeling.
\subsubsection{Effect of Weak Teacher}
We conduct ablation studies on the stage of weak teachers adopted in our method under the setting of ImageNet-1K with IPC 10, IPC 50, and IPC 100. The results are shown in Fig.~\ref{fig:motivate}. Here, we generate the model pool with weak teachers in different stages. 
We cache 9 models for each trajectory segment with step 5. For data generation, we randomly select 3 out of the model pool and use the whole models in the model pool to relabel the synthetic data. From the results, it is evident that the performance of teacher models significantly impacts the distilled data. For smaller IPCs, adopting the weak teacher models of the early stage is necessary.

\subsubsection{Size of the Model Pool}
We also perform an ablation study on the size of the model pool, and the results are shown in Fig.~\ref{fig:poolsize}. The experiments are under the setting of ImageNet-1K IPC 10 and IPC 50. 
For one experiment setting, we cache the teacher models from the same trajectory segments; the beginning and ending points are the same as the main table but with different step lengths.
Here, we randomly select 3 models from the model pool to generate data, and the whole model pool is adopted to generate the soft label. 
From the results, we observe that the size of the model pool does not significantly impact the performance of the generated synthetic data when the size is larger than 5. In this setting, a smaller size of the model pool 
may lead to an issue with mode collapse and a subsequent decline in performance.

\subsection{Efficiency Evaluation}
In this section, we evaluate the efficiency of our proposed method. The experiments are under the setting of ImageNet-1K IPC 10, utilizing 8 RTX 4090 GPUs. 
% The size of the model pool is 5 and 9, and the cached teachers are at stages 1 to 41. 
Here, the cached teachers are at stages 1 to 41.
We report the total run time of model pre-training and data synthesizing, and the peak GPU memory of the data synthesizing process. The evaluation results are shown in Fig.~\ref{fig:efficiency}. 
% The size of the points represents the peak GPU memory costs, and the three points in Fig.~\ref{fig:efficiency}, from left to right, report the evaluation results of 1, 2, and 3 teacher models ensemble utilized in generating the synthetic data. 
From the results, due to the fact that we only need to train teacher models for a limited number of epochs~(which is the major time overhead), the total time required for our method is significantly less than that of SRe$^2$L. Noticing that when utilizing a single teacher model, the required GPU memory is the same as SRe$^2$L, and the performance of the distilled data is notably superior to SRe$^2$L.

\subsection{Continual Learning}
Continual learning~(CL)~\cite{de2021continual,wang2024comprehensive,rebuffi2017icarl} focuses on learning tasks sequentially while prior task data is unavailable. In order to prevent catastrophic forgetting, one typical strategy is using a small buffer to retain critical information from past tasks. In this case, the capacity of dataset distillation to retain the information of the original datasets within a confined storage space benefits practical applications in continual learning. 

Here, we follow the previous work~\cite{zhao2023datasetd} employing the GDumb~\cite{prabhu2020gdumb} as the base method.
This method stores and combines past and new training data to train a model from scratch.
We evaluate our proposed method on Tiny-ImageNet with IPC 50 following the 5-step class-incremental protocol. The results are shown in Fig.~\ref{fig:continual}.
Our method significantly outperforms the baseline methods and demonstrates an accuracy trend that differs from random selection, gradually increasing with the addition of categories. This trend is attributed to the soft label strategy, which allows for incorporating more cross-class knowledge.

\begin{figure}[t]
  \centering
   \includegraphics[width=\linewidth]{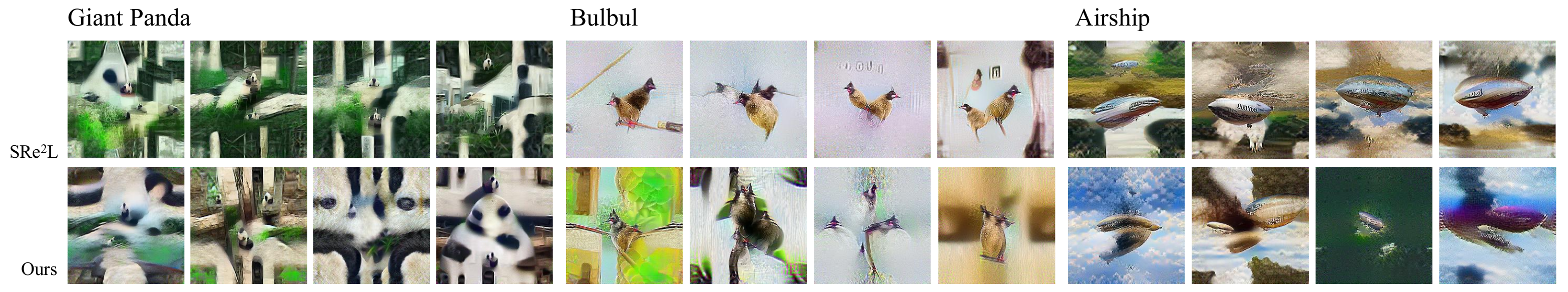}
   \caption{Visualization of synthetic data generated by SRe$^2$L and our method. The first row is generated by SRe$^2$L, and the second row is generated by ours with the prior-generation strategy. Here, we choose three classes from ImageNet-1K: Giant Panda, Bulbul, and Airship.}
   \label{fig:vis}
\end{figure}
\subsection{Visualization}
In this section, we present the visualization of ImageNet-1K synthetic datasets with IPC 50 setting generated by our method and SRe$^2$L~\cite{yin2023squeeze}. Fig.~\ref{fig:vis} shows that our method generates high quality, more distinctive features and a more comprehensive range of variations data, showing more perspective of the original dataset. In comparison, images generated from SRe$^2$L, while realistic, face the problem of mode collapse, especially for large IPCs.

\section{Conclusion}
In this paper, we introduce Teddy, an efficient large-scale dataset distillation method that can handle the full-size ImageNet-1K. 
Back up with theoretical analysis, we propose a memory-efficient approximation derived from Taylor expansion to decouple the bi-level optimization of the original DD into a first-order one. 
To further improve the time efficiency, instead of training a novel model for each synthetic data update, we propose two efficient model pool generation strategies, the prior and post-generation, corresponding to two practical cases. 
Both strategies can generate the model pool from one single base model.
Lastly, we adopt ensemble soft relabeling to improve the generalizability of the distilled data. 
Additionally, to the best of our knowledge, we are the first to build a unified theoretical framework of the DD, linking the existing mainstream DD methods.
Various experiments and ablation studies show the effectiveness of our proposed method, outperforming the prior state-of-the-art DD methods by a significant margin.

\section*{Acknowledgements}
This project is supported by the Ministry of Education, Singapore, under its Academic Research Fund Tier 2 (Award Number: MOE-T2EP20122-0006),
and the National Research Foundation, Singapore, under its AI Singapore Programme (AISG Award No: AISG2-RP-2021-023).

% ---- Bibliography ----
%
% BibTeX users should specify bibliography style 'splncs04'.
% References will then be sorted and formatted in the correct style.
%
\bibliographystyle{splncs04}
\bibliography{main}

\begin{thebibliography}{10}
\providecommand{\url}[1]{\texttt{#1}}
\providecommand{\urlprefix}{URL }
\providecommand{\doi}[1]{https://doi.org/#1}

\bibitem{bohdal2020flexible}
Bohdal, O., Yang, Y., Hospedales, T.: Flexible dataset distillation: Learn labels instead of images. arXiv preprint arXiv:2006.08572  (2020)

\bibitem{bubeck2023sparks}
Bubeck, S., Chandrasekaran, V., Eldan, R., Gehrke, J., Horvitz, E., Kamar, E., Lee, P., Lee, Y.T., Li, Y., Lundberg, S., et~al.: Sparks of artificial general intelligence: Early experiments with gpt-4. arXiv preprint arXiv:2303.12712  (2023)

\bibitem{cazenavette2022dataset}
Cazenavette, G., Wang, T., Torralba, A., Efros, A.A., Zhu, J.Y.: Dataset distillation by matching training trajectories. In: Proceedings of the IEEE/CVF Conference on Computer Vision and Pattern Recognition. pp. 4750--4759 (2022)

\bibitem{cui2022dc}
Cui, J., Wang, R., Si, S., Hsieh, C.J.: Dc-bench: Dataset condensation benchmark. Advances in Neural Information Processing Systems  \textbf{35},  810--822 (2022)

\bibitem{cui2023scaling}
Cui, J., Wang, R., Si, S., Hsieh, C.J.: Scaling up dataset distillation to imagenet-1k with constant memory. In: International Conference on Machine Learning. pp. 6565--6590. PMLR (2023)

\bibitem{de2021continual}
De~Lange, M., Aljundi, R., Masana, M., Parisot, S., Jia, X., Leonardis, A., Slabaugh, G., Tuytelaars, T.: A continual learning survey: Defying forgetting in classification tasks. IEEE transactions on pattern analysis and machine intelligence  \textbf{44}(7),  3366--3385 (2021)

\bibitem{deng2009imagenet}
Deng, J., Dong, W., Socher, R., Li, L.J., Li, K., Fei-Fei, L.: Imagenet: A large-scale hierarchical image database. In: 2009 IEEE conference on computer vision and pattern recognition. pp. 248--255. Ieee (2009)

\bibitem{deng2022remember}
Deng, Z., Russakovsky, O.: Remember the past: Distilling datasets into addressable memories for neural networks. Advances in Neural Information Processing Systems  \textbf{35},  34391--34404 (2022)

\bibitem{dosovitskiy2020image}
Dosovitskiy, A., Beyer, L., Kolesnikov, A., Weissenborn, D., Zhai, X., Unterthiner, T., Dehghani, M., Minderer, M., Heigold, G., Gelly, S., et~al.: An image is worth 16x16 words: Transformers for image recognition at scale. arXiv preprint arXiv:2010.11929  (2020)

\bibitem{du2023minimizing}
Du, J., Jiang, Y., Tan, V.Y., Zhou, J.T., Li, H.: Minimizing the accumulated trajectory error to improve dataset distillation. In: Proceedings of the IEEE/CVF Conference on Computer Vision and Pattern Recognition. pp. 3749--3758 (2023)

\bibitem{fang2023depgraph}
Fang, G., Ma, X., Song, M., Mi, M.B., Wang, X.: Depgraph: Towards any structural pruning. In: Proceedings of the IEEE/CVF Conference on Computer Vision and Pattern Recognition. pp. 16091--16101 (2023)

\bibitem{he2020momentum}
He, K., Fan, H., Wu, Y., Xie, S., Girshick, R.: Momentum contrast for unsupervised visual representation learning. In: Proceedings of the IEEE/CVF conference on computer vision and pattern recognition. pp. 9729--9738 (2020)

\bibitem{he2016deep}
He, K., Zhang, X., Ren, S., Sun, J.: Deep residual learning for image recognition. In: Proceedings of the IEEE conference on computer vision and pattern recognition. pp. 770--778 (2016)

\bibitem{huang2017densely}
Huang, G., Liu, Z., Van Der~Maaten, L., Weinberger, K.Q.: Densely connected convolutional networks. In: Proceedings of the IEEE conference on computer vision and pattern recognition. pp. 4700--4708 (2017)

\bibitem{kim2022dataset}
Kim, J.H., Kim, J., Oh, S.J., Yun, S., Song, H., Jeong, J., Ha, J.W., Song, H.O.: Dataset condensation via efficient synthetic-data parameterization. In: International Conference on Machine Learning. pp. 11102--11118. PMLR (2022)

\bibitem{krizhevsky2012imagenet}
Krizhevsky, A., Sutskever, I., Hinton, G.E.: Imagenet classification with deep convolutional neural networks. Advances in neural information processing systems  \textbf{25} (2012)

\bibitem{lee2022dataset}
Lee, H.B., Lee, D.B., Hwang, S.J.: Dataset condensation with latent space knowledge factorization and sharing. arXiv preprint arXiv:2208.10494  (2022)

\bibitem{liu2022dataset}
Liu, S., Wang, K., Yang, X., Ye, J., Wang, X.: Dataset distillation via factorization. Advances in Neural Information Processing Systems  \textbf{35},  1100--1113 (2022)

\bibitem{liu2024mgdd}
Liu, S., Wang, X.: Mgdd: A meta generator for fast dataset distillation. Advances in Neural Information Processing Systems  \textbf{36} (2024)

\bibitem{liu2023slimmable}
Liu, S., Ye, J., Yu, R., Wang, X.: Slimmable dataset condensation. In: Proceedings of the IEEE/CVF Conference on Computer Vision and Pattern Recognition. pp. 3759--3768 (2023)

\bibitem{ma2018shufflenet}
Ma, N., Zhang, X., Zheng, H.T., Sun, J.: Shufflenet v2: Practical guidelines for efficient cnn architecture design. In: Proceedings of the European conference on computer vision (ECCV). pp. 116--131 (2018)

\bibitem{nguyen2020dataset}
Nguyen, T., Chen, Z., Lee, J.: Dataset meta-learning from kernel ridge-regression. arXiv preprint arXiv:2011.00050  (2020)

\bibitem{prabhu2020gdumb}
Prabhu, A., Torr, P.H., Dokania, P.K.: Gdumb: A simple approach that questions our progress in continual learning. In: Computer Vision--ECCV 2020: 16th European Conference, Glasgow, UK, August 23--28, 2020, Proceedings, Part II 16. pp. 524--540. Springer (2020)

\bibitem{ramesh2022hierarchical}
Ramesh, A., Dhariwal, P., Nichol, A., Chu, C., Chen, M.: Hierarchical text-conditional image generation with clip latents. arXiv preprint arXiv:2204.06125  \textbf{1}(2), ~3 (2022)

\bibitem{rebuffi2017icarl}
Rebuffi, S.A., Kolesnikov, A., Sperl, G., Lampert, C.H.: icarl: Incremental classifier and representation learning. In: Proceedings of the IEEE conference on Computer Vision and Pattern Recognition. pp. 2001--2010 (2017)

\bibitem{rombach2022high}
Rombach, R., Blattmann, A., Lorenz, D., Esser, P., Ommer, B.: High-resolution image synthesis with latent diffusion models. In: Proceedings of the IEEE/CVF conference on computer vision and pattern recognition. pp. 10684--10695 (2022)

\bibitem{sandler2018mobilenetv2}
Sandler, M., Howard, A., Zhu, M., Zhmoginov, A., Chen, L.C.: Mobilenetv2: Inverted residuals and linear bottlenecks. In: Proceedings of the IEEE conference on computer vision and pattern recognition. pp. 4510--4520 (2018)

\bibitem{simonyan2014very}
Simonyan, K., Zisserman, A.: Very deep convolutional networks for large-scale image recognition. arXiv preprint arXiv:1409.1556  (2014)

\bibitem{sucholutsky2021soft}
Sucholutsky, I., Schonlau, M.: Soft-label dataset distillation and text dataset distillation. In: 2021 International Joint Conference on Neural Networks (IJCNN). pp.~1--8. IEEE (2021)

\bibitem{tan2019efficientnet}
Tan, M., Le, Q.: Efficientnet: Rethinking model scaling for convolutional neural networks. In: International conference on machine learning. pp. 6105--6114. PMLR (2019)

\bibitem{wang2022cafe}
Wang, K., Zhao, B., Peng, X., Zhu, Z., Yang, S., Wang, S., Huang, G., Bilen, H., Wang, X., You, Y.: Cafe: Learning to condense dataset by aligning features. In: Proceedings of the IEEE/CVF Conference on Computer Vision and Pattern Recognition. pp. 12196--12205 (2022)

\bibitem{wang2024comprehensive}
Wang, L., Zhang, X., Su, H., Zhu, J.: A comprehensive survey of continual learning: Theory, method and application. IEEE Transactions on Pattern Analysis and Machine Intelligence  (2024)

\bibitem{wang2018dataset}
Wang, T., Zhu, J.Y., Torralba, A., Efros, A.A.: Dataset distillation. arXiv preprint arXiv:1811.10959  (2018)

\bibitem{ye2024distilled}
Ye, J., Yu, R., Liu, S., Wang, X.: Distilled datamodel with reverse gradient matching. In: Proceedings of the IEEE/CVF Conference on Computer Vision and Pattern Recognition. pp. 11954--11963 (2024)

\bibitem{ye2024mutual}
Ye, J., Yu, R., Liu, S., Wang, X.: Mutual-modality adversarial attack with semantic perturbation. In: Proceedings of the AAAI Conference on Artificial Intelligence. vol.~38, pp. 6657--6665 (2024)

\bibitem{yin2023squeeze}
Yin, Z., Xing, E., Shen, Z.: Squeeze, recover and relabel: Dataset condensation at imagenet scale from a new perspective. arXiv preprint arXiv:2306.13092  (2023)

\bibitem{yu2023dataset}
Yu, R., Liu, S., Wang, X.: Dataset distillation: A comprehensive review. arXiv preprint arXiv:2301.07014  (2023)

\bibitem{zhang2023accelerating}
Zhang, L., Zhang, J., Lei, B., Mukherjee, S., Pan, X., Zhao, B., Ding, C., Li, Y., Xu, D.: Accelerating dataset distillation via model augmentation. In: Proceedings of the IEEE/CVF Conference on Computer Vision and Pattern Recognition. pp. 11950--11959 (2023)

\bibitem{zhao2021datasetdsa}
Zhao, B., Bilen, H.: Dataset condensation with differentiable siamese augmentation. In: International Conference on Machine Learning. pp. 12674--12685. PMLR (2021)

\bibitem{zhao2023datasetd}
Zhao, B., Bilen, H.: Dataset condensation with distribution matching. In: Proceedings of the IEEE/CVF Winter Conference on Applications of Computer Vision. pp. 6514--6523 (2023)

\bibitem{zhao2020dataset}
Zhao, B., Mopuri, K.R., Bilen, H.: Dataset condensation with gradient matching. arXiv preprint arXiv:2006.05929  (2020)

\bibitem{zhao2023improved}
Zhao, G., Li, G., Qin, Y., Yu, Y.: Improved distribution matching for dataset condensation. In: Proceedings of the IEEE/CVF Conference on Computer Vision and Pattern Recognition. pp. 7856--7865 (2023)

\bibitem{zhou2022dataset}
Zhou, Y., Nezhadarya, E., Ba, J.: Dataset distillation using neural feature regression. Advances in Neural Information Processing Systems  \textbf{35},  9813--9827 (2022)

\end{thebibliography}
\clearpage
\appendix
\section{Comprehensive Theoretical Analysis}
Given the original dataset, denoted as $\mathcal{T}=(X_t, Y_t)$, where $X_t \in \mathbb{R}^{N_t\times d}$ and $Y_t \in \mathbb{R}^{N_t \times c}$, DD aims to learn a much smaller synthetic dataset $\mathcal{S}=(X_s, Y_s)$, where $X_s \in \mathbb{R}^{N_s\times d}$ and $Y_s \in \mathbb{R}^{N_s \times c}$, such that models trained on $\mathcal{S}$ and $\mathcal{T}$ are of comparable performance.
Here, $N_t$ and $N_s$ are the number of data in $\mathcal{T}$ and $\mathcal{S}$, respectively, and $N_s \ll N_t$. $d$ is number of features for each data, and $c$ represents the number of classes for classification task. 

For better generalization ability of $\mathcal{S}$, it always adopts multiple inner loops to train a novel student model. Considering the small size of $\mathcal{S}$, models trained on it can easily overfit. Thus, we have $l(\mathcal{S};\theta^{(T)}_{\mathcal{S}})<\epsilon$, where $l(\cdot;\theta)$ is the loss function, $T$ is the number of inner loops, and $\epsilon$ is close to zero.
With the Karush-Kuhn-Tucker~(KKT) conditions, the problem can be formulated as:
\begin{equation}
    \begin{split}
    \mathcal{S} &= \mathop{\arg\min}\limits_{\mathcal{S}}\mathbb{E}_{\theta^{(0)}\sim\Theta}[l(\mathcal{T};\theta^{(T)}_{\mathcal{S}}) +u( l(\mathcal{S};\theta^{(T)}_{\mathcal{S}})-\epsilon)] \\    &=\mathop{\arg\min}\limits_{\mathcal{S}}\mathbb{E}_{\theta^{(0)}\sim\Theta}[l(\mathcal{T};\theta^{(T)}_{\mathcal{S}})+ul(\mathcal{S};\theta^{(T)}_{\mathcal{S}})],\\
\end{split}
\end{equation}
where $u$ is the Lagrange multiplier, $\Theta$ is the distribution of networks, and 
\begin{equation}
\begin{split}
    \theta^{(t)}_{\mathcal{S}} &= \theta^{(0)} -\alpha\sum_{i=0}^{t-1}g_{\mathcal{S}}^{(i)} = \theta^{(t-1)}_{\mathcal{S}}-\alpha g^{(t-1)}_{\mathcal{S}},\\
    g_{\mathcal{S}}^{(t)}&=\nabla_{\theta^{(t)}_{\mathcal{S}}} l(\mathcal{S}; \theta^{(t)}_{\mathcal{S}}).
\end{split}
\end{equation}
Following the Proposition~1, we recursively apply the first-order Taylor expansion to the first term of the optimization objective, and it can be transformed into follows:
\begin{equation}
\label{eq:iter-tay}
    \begin{split}
        l(\mathcal{T};\theta^{(T)}_{\mathcal{S}}) 
     &= l(\mathcal{T};\theta^{(T-1)}_{\mathcal{S}}-\alpha g_{\mathcal{S}}^{(T-1)}) \\
     &= l(\mathcal{T};\theta^{(T-1)}_{\mathcal{S}})-\alpha g_{\mathcal{T}}^{(T-1)}\cdot g_{\mathcal{S}}^{(T-1)} \\
     &= l(\mathcal{T};\theta^{(T-i)}_{\mathcal{S}})-\alpha \sum_{t=T-i}^{T-1}g_{\mathcal{T}}^{(t)}\cdot g_{\mathcal{S}}^{(t)} \\
     &= l(\mathcal{T};\theta^{(0)})-\alpha \sum_{t=0}^{T-1}g_{\mathcal{T}}^{(t)}\cdot g_{\mathcal{S}}^{(t)},
    \end{split}
\end{equation}
where $g_{\mathcal{T}}^{(t)}=\nabla_{\theta^{(t)}_{\mathcal{S}}} l(\mathcal{T}; \theta^{(t)}_{\mathcal{S}})$. The Taylor-approximated version of the optimization objective is as follows:
\begin{equation}
    \begin{split}
    \mathcal{S}&=\mathop{\arg\min}\limits_{\mathcal{S}}[-\alpha \sum_{t=0}^{T-1}g_{\mathcal{T}}^{(t)}\cdot g_{\mathcal{S}}^{(t)}+ul(\mathcal{S};\theta^{(T)}_{\mathcal{S}})]\\
    &=\mathop{\arg\min}\limits_{\mathcal{S}}[-\sum_{t=0}^{T-1}\frac{g_{\mathcal{T}}^{(t)}\cdot g_{\mathcal{S}}^{(t)}}{||g_{\mathcal{T}}^{(t)}||||g_{\mathcal{S}}^{(t)}||}+u^{'} l(\mathcal{S};\theta^{(T)}_{\mathcal{S}})],
\end{split}
\label{eq:taylor-obj}
\end{equation}
where $\theta^{(0)}\sim\Theta$, and $u^{'}=\frac{u}{\alpha||g_{\mathcal{T}}^{(t)}||||g_{\mathcal{S}}^{(t)}||}$. The first term of the Eq.~\ref{eq:taylor-obj} is the sum of the cosine distance of the gradients for the synthetic and original data, respectively, on the all checkpoints of the student trajectory. In other words, the meta-learning-based optimization objective can be Taylor-approximated as the sum of the multi-step gradient matching. Here, we denote $e^{(t)}=\frac{g^{(t)}}{||g^{(t)}||}$, and for each gradient matching at step $t$, we have:
\begin{equation}
    \begin{split}
        - e^{(t)}_{\mathcal{T}}\cdot e^{(t)}_{\mathcal{S}}=\frac{1}{2}||e^{(t)}_{\mathcal{T}} - e^{(t)}_{\mathcal{S}}||^2-1.
    \end{split}
\end{equation}
So for each step of gradient matching, we only need to optimize the $l^2$ distance between the gradients of the original data and the synthetic data, such that:
\begin{equation}
    \begin{split}
        \mathcal{S}&=\mathop{\arg\min}\limits_{\mathcal{S}}[\frac{1}{2}\sum_{t=0}^{T-1}||e^{(t)}_{\mathcal{T}} - e^{(t)}_{\mathcal{S}}||^2+u^{'}l(\mathcal{S};\theta^{(T)}_{\mathcal{S}})].\\
    \end{split}
\label{eq:e-dis}
\end{equation}
However, the Eq.~\ref{eq:e-dis} shows that each update step of $\mathcal{S}$ should compute the gradient of $\mathcal{S}$ and $\mathcal{T}$ on all checkpoints of the student trajectory, which is a large amount of computational costs. Here, we reduce the costs by adopting another approximation strategy, turning the second-order optimization into the first-order one. 
Following the Proposition~2, for simplicity in explanation, we only consider the last linear layer of the network is updated during the generation process, and the parameter is denoted as $W$. The preceding layers are regraded as the feature extractor, denoted as $f_{\theta}$. For each gradient, we have:
\begin{equation}
\begin{split}
    g &= \frac{\partial l(\cdot;\theta)}{\partial W} =\frac{1}{|X|} f_{\theta}(X)^T(f_{\theta}(X)W-Y)\\
      &= \frac{1}{|X|}f_{\theta}(X)^Tf_{\theta}(X)W - \frac{1}{|X|}f_{\theta}(X)^TY.
\end{split}
\label{eq:gradtosec}
\end{equation}
The first term of Eq.~\ref{eq:gradtosec} is a weighted covariance matrix of the feature space, and the second term is the class-wise mean of the feature space. The gradient matching for $l^2$ distance can be transformed as:
\begin{equation}
    \begin{split}
        &||(\frac{1}{N_t}f_{\theta}(X_t)^Tf_{\theta}(X_t)-\frac{1}{N_s}f_{\theta}(X_s)^Tf_{\theta}(X_s))W \\
        &- (\frac{1}{N_t}f_{\theta}(X_t)^TY - \frac{1}{N_s}f_{\theta}(X_s)^TY)||^2\frac{1}{||W||^2}\\
        \leq& ||\frac{1}{N_t}f_{\theta}(X_t)^Tf_{\theta}(X_t)-\frac{1}{N_s}f_{\theta}(X_s)^Tf_{\theta}(X_s)||^2\\
        &+||\frac{1}{N_t}f_{\theta}(X_t)^TY - \frac{1}{N_s}f_{\theta}(X_s)^TY||^2\frac{1}{||W||^2}.
    \end{split}
    \label{eq:gradupper}
\end{equation}
From Eq.~\ref{eq:gradupper}, it shows that the upper bound of the gradient matching is the first-order and the second-order statistic information matching in feature space. Here, we need to consider the prior condition such that $\theta^{(0)}\sim\Theta$.
It means that for every $\theta^{(0)}$ samples from the distribution $\Theta$, Eq.~\ref{eq:gradupper} reaches the minimum value. In this case, the first term of Eq.~\ref{eq:gradupper}, the covariance for the original data and synthetic data, is equivalent for all $\theta^{(0)}$ samples, and the class-wise mean. Also, the equality in Eq.~\ref{eq:gradupper} holds. As for $\theta^{(t)}$, as long as the samples are sufficient, the condition still holds. To be more specific, here we assume $f_{\theta}(X_t)\in\mathbb{R}^{N_t\times f_d}$, and $f_{\theta}(X_s)\in\mathbb{R}^{N_s\times f_d}$, $f_d$ is the dimension of the feature. 
$f_{\theta}(X_t)=[f_{\theta}(x^t_1)^T,\dots,f_{\theta}(x^t_{N_t})^T]^T$, and $f_{\theta}(x^t_i)\in\mathbb{R}^{f_d}$. Also, $f_{\theta}(X_s)=[f_{\theta}(x^s_1)^T,\dots,f_{\theta}(x^s_{N_s})^T]^T$, and $f_{\theta}(x^s_i)\in\mathbb{R}^{f_d}$. $f_{\theta}(X_t)=[F^t_1,\dots,F^t_{f_d}]$, $F^t_i\in\mathbb{R}^{N_t}$, and $f_{\theta}(X_s)=[F^s_1,\dots,F^s_{f_d}]$, $F^s_i\in\mathbb{R}^{N_s}$. For the first term of Eq.~\ref{eq:gradupper}, we have:
\begin{equation}
    \begin{split}
        &\frac{1}{N_t}f_{\theta}(X_t)^Tf_{\theta}(X_t)-\frac{1}{N_s}f_{\theta}(X_s)^Tf_{\theta}(X_s)\\
        =&\frac{1}{N_t}[(F^t_1)^T,\dots,(F^t_{f_d})^T]^T[F^t_1,\dots,F^t_{f_d}]\\
        &-\frac{1}{N_s}[(F^s_1)^T,\dots,(F^s_{f_d})^T]^T[F^s_1,\dots,F^s_{f_d}]\\
        =&\begin{bmatrix}
            Var(F^t_1)&\cdots&Cov(F^t_1, F^t_{f_d})\\
            \vdots&\ddots&\vdots\\
            Cov(F^t_{f_d}, F^t_1)&\cdots&Var(F^t_{f_d})\\
        \end{bmatrix}\\
        &-\begin{bmatrix}
            Var(F^s_1)&\cdots&Cov(F^s_1, F^s_{f_d})\\
            \vdots&\ddots&\vdots\\
            Cov(F^s_{f_d}, F^s_1)&\cdots&Var(F^s_{f_d})\\
        \end{bmatrix}.\\
    \end{split}
    \label{eq:unfold}
\end{equation}
Eq.~\ref{eq:unfold} shows that if the covariance matching holds, then variance matching also holds. To improve the calculation efficiency, here we adopt variance matching instead of covariance matching. As for the second term, it is the class-wise mean matching in feature space. Specifically, we have: 
\begin{equation}
    \begin{split}
        &\frac{1}{N_t}f_{\theta}(X_t)^TY - \frac{1}{N_s}f_{\theta}(X_s)^TY\\
        =&\frac{1}{N_t}[f_{\theta}(x^t_1),\dots,f_{\theta}(x^t_{N_t})]\mathcal{M}_c^t\\
        &-\frac{1}{N_s}[f_{\theta}(x^s_1),\dots,f_{\theta}(x^s_{N_s})]\mathcal{M}_c^s.
    \end{split}
    \label{eq:classmean}
\end{equation}
Here, $\mathcal{M}_c^t\in\mathbb{R}^{N_t\times c}$ and $\mathcal{M}_c^s\in\mathbb{R}^{N_s\times c}$, are the matrices to indicate whether the samples belong to the classes. For instance, $M_{ij}=1$ means the $i^{th}$ sample is belong to the $j^{th}$ class, otherwise, it is not. The Eq.~\ref{eq:classmean} can be:
\begin{equation}
\label{eq:balance}
\begin{split}
    &\frac{1}{N_t}[\sum_{i=1}^{N_t}f_{\theta}(x^t_i)M_{i1}^t,\dots,\sum_{i=1}^{N_t}f_{\theta}(x^t_i)M_{ic}^t]\\
    &-\frac{1}{N_s}[\sum_{i=1}^{N_s}f_{\theta}(x^s_i)M_{i1}^s,\dots,\sum_{i=1}^{N_s}f_{\theta}(x^s_i)M_{ic}^s]\\
    =&\frac{1}{N_t}[\sum_{x_i^t\in Cl_1}f_{\theta}(x^t_i),\dots,\sum_{x_i^t\in Cl_c}f_{\theta}(x^t_i)]\\
    &-\frac{1}{N_s}[\sum_{x_i^s\in Cl_1}f_{\theta}(x^s_i),\dots,\sum_{x_i^s\in Cl_c}f_{\theta}(x^s_i)],
\end{split}
\end{equation}
where $Cl_i$ is the set of samples belong to $i^{th}$ class. When the classes of the original dataset is balanced, or the number of samples in every class of the original dataset is same, the second term of Eq.~\ref{eq:gradupper} can be replaced by the global mean.
Therefore, the Eq.~\ref{eq:e-dis} can be transformed as follows:
\begin{equation}
\begin{split}
    &\sum_{t=0}^{T-1}(\sum_l||\mu_l(f_{\theta_{\mathcal{S}}^{(t)}}(X_s))-\mu_l(f_{\theta_{\mathcal{S}}^{(t)}}(X_t))||_2 \\
     &+\sum_l ||\sigma_l^2(f_{\theta_{\mathcal{S}}^{(t)}}(X_s))-\sigma_l^2(f_{\theta_{\mathcal{S}}^{(t)}}(X_t))||_2) \\
     &+u\cdot l(\mathcal{S};\theta^{(T)}_{\mathcal{S}}),
\end{split}
\label{eq:stu}
\end{equation}
where $\theta^{(0)}\sim\Theta$, $\mu_l$ and $\sigma^2_l$ refer to the mean and variance of the $l^{th}$ layer features. Although Eq.~\ref{eq:stu} has significantly improved efficiency, computing the mean and variance of checkpoints at each inner loop remains a substantial computational overhead, especially when there are numerous inner loops and a relatively large original dataset. Here, we propose substituting performance-comparable weak teachers for student models and reducing the number of rounds required for inner loop.
As for the student trajectory starting at $t$ and ending at $t+m$, we have:
\begin{equation}
    \begin{split} l(\mathcal{T};\theta_{\mathcal{S}}^{(t+m)})&=l(\mathcal{T};\theta_{\mathcal{S}}^{(t)}-\alpha\sum_{i=t}^{t+m-1}g_{\mathcal{S}}^{(i)}) \\
    &=l(\mathcal{T};\theta^{(t)}_{\mathcal{S}})-\alpha (\sum_{i=t}^{t+m-1}g_{\mathcal{S}}^{(i)})\cdot g_{\mathcal{T}}^{(t)}.
    \end{split}
    \label{eq:multi-step2}
\end{equation}
From Eq.~\ref{eq:multi-step2}, it shows that the multi-step gradient of the model on the synthetic dataset compared to the single-step gradient on the original dataset. 
Therefore, we cache the weak teacher for every $m$ steps from $T_b$ to $T_e$. 
Here, we also utilize the running mean and running variance stored in the batch normalization layers of weak teachers to replace the mean and variance of the original dataset in feature space, reducing computational costs. The optimization objective is as follows:
\begin{equation}
\begin{split}
    &\sum_{t=T_b}^{T_e}(\sum_l||\mu_l(f_{\theta_{\mathcal{T}}^{(t)}}(X_s))-RM_{\theta^{(t)}_{\mathcal{T}}}^l(X_t))||_2 \\
     &+\sum_l ||\sigma_l^2(f_{\theta_{\mathcal{T}}^{(t)}}(X_s))-RV_{\theta^{(t)}_{\mathcal{T}}}^l(X_t))||_2 \\
     &+u\cdot l(\mathcal{S};\theta^{(t)}_{\mathcal{T}})),
\end{split}
\end{equation}
where $T_b$ and $T_e$ is the starting checkpoint and the end point of the teacher training trajectory.

\section{Error Analysis for Taylor Approximation}
In Eq.~\ref{eq:iter-tay}, we recursively apply the first-order Taylor expansion to the original optimization objective to decouple the bi-level optimization process. This section will further analyze the error caused by Taylor approximation part, and theoretical reasonability.

\begin{prop}
    Taylor approximation minimizes the upper bound of the original loss function. 
\end{prop}
\begin{proof}
Considering each Taylor approximation iteration independently is as follows: 
\begin{equation}
\begin{split}
    l(\mathcal{T};\theta_{\mathcal{S}}^{(t)})&=||f(X_t;\theta_{\mathcal{S}}^{(t)})-Y_t|| \\
    &\leq||f(X_t;\theta_{\mathcal{S}}^{(t)})-f(X_t;\theta_{\mathcal{T}}^{(t)})||+||f(X_t;\theta_{\mathcal{T}}^{(t)})-Y_t|| \\
    &\leq \beta||\theta_{\mathcal{S}}^{(t)}-\theta_{\mathcal{T}}^{(t)}||+||f(X_t;\theta_{\mathcal{T}}^{(t)})-Y_t|| \\
    &=\beta||g_{\mathcal{S}}^{(t-1)}-g_{\mathcal{T}}^{(t-1)}|| + C \\
    &=\beta g_{\mathcal{S}}^{(t-1)}\cdot g_{\mathcal{T}}^{(t-1)} + C,
\end{split}
\end{equation}
where $\beta$ is the Lipschitz constant and $C$ is irrelevant to the optimization. 
\end{proof}

We also conduct validation experiments under the setting of CIFAR10 IPC 5 with the original DD optimization objective and our Taylor-approximation version. The results are shown in Fig.~\ref{fig:R1}. We evaluate the difference of average losses, average accuracy per iteration, and peak accuracy during training. The results demonstrate that the bound is tight in practice, and indicate that errors introduced by our approximation are negligible and the overall training dynamics are comparable.

\begin{figure}[t]
   \includegraphics[width=\linewidth]{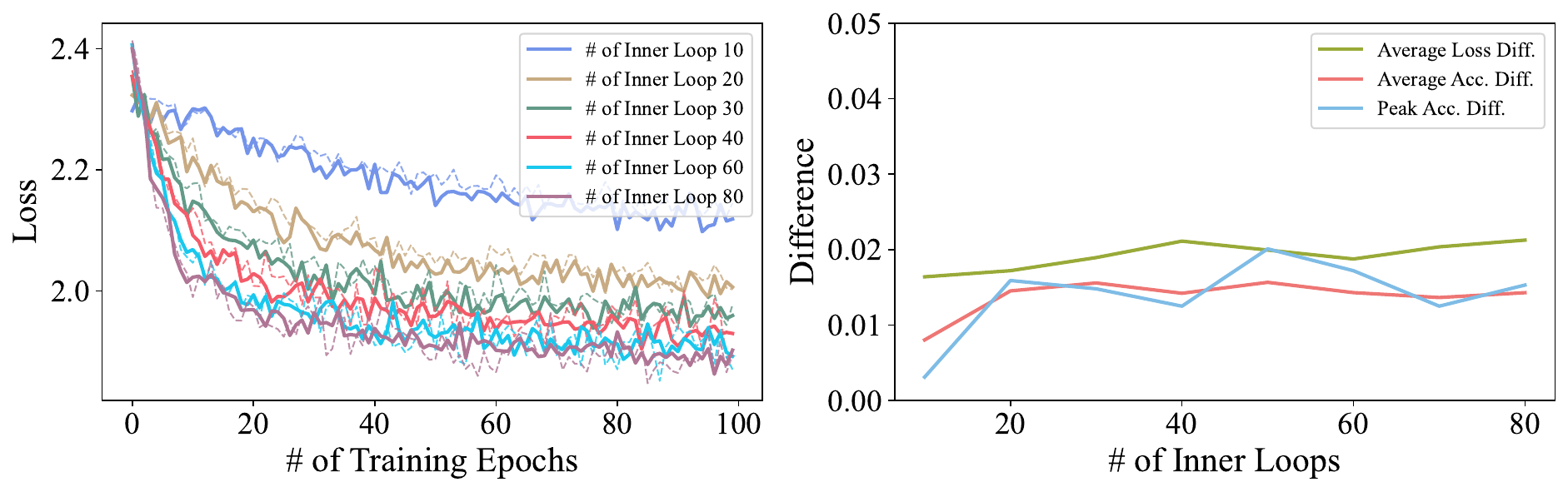}
   \caption{Left: training loss of the original DD~(dashed line) and our approximated objectives~(solid line); Right: the difference of average loss, average accuracy and peak accuracy during training.}
   \label{fig:R1}
\end{figure}

\begin{table}[!t]
    \centering
    \renewcommand{\arraystretch}{0.8}
    \resizebox{\linewidth}{!}{
    \begin{tabular}{c c c c }
        \toprule
           \bf Imbalanced dataset Acc.(\%) / F1(\%) & \bf SRe$^2$L & \bf Ours & \bf Ours with Class-wise Statistic \\
            \midrule
            \bf 60\%$\sim$100\% & 24.2 / 24.2 & \bf 32.7~(\textcolor{blue}{+ 8.5}) / 32.5~(\textcolor{blue}{+ 8.3}) & \bf 36.0~(\textcolor{blue}{+ 11.8}) / 36.0~(\textcolor{blue}{+ 11.8})\\
            \bf 40\%$\sim$100\% & 23.4 / 23.0 & \bf 33.6~(\textcolor{blue}{+ 10.2}) / 33.3~(\textcolor{blue}{+ 10.3}) & \bf 35.6~(\textcolor{blue}{+ 12.2}) / 35.3~(\textcolor{blue}{+ 12.3})\\       
        \bottomrule
    \end{tabular}
    }
    \caption{Results on imbalanced dataset with each class has 60\%-100\% and 40\%-100\% samples.}
    \label{tab:R2}
\end{table}
\begin{table}[!t]
    \centering
    \renewcommand{\arraystretch}{0.8}
    \resizebox{\linewidth}{!}{
    \begin{tabular}{c c c c  |c c c c}
        \toprule
        \bf Acc.~(\%) & \bf 5 & \bf 10 & \bf 50 & \bf mAP~(\%) & \bf 10 Imgs & \bf 20 Imgs & \bf 40 Imgs\\
        \midrule
        \bf SRe$^2$L & 5.4 & 17.8 & 48.4 & \bf baseline & 3.6 & 5.6 & 9.9\\
        \bf Ours & \bf 21.0~(\textcolor{blue}{+ 15.6}) & \bf 37.2~(\textcolor{blue}{+ 19.4}) & \bf 58.5~(\textcolor{blue}{+ 10.1}) & \bf ours &  6.4~(\textcolor{blue}{+ 2.8}) & \bf 10.4 (\textcolor{blue}{+ 4.8}) & \bf 15.3~(\textcolor{blue}{+ 5.4}) \\
        \bottomrule
    \end{tabular}
    }
    \caption{Results on larger model~(left) and other task~(right).}
    \label{tab:R3}
\end{table}

\section{More Experimental Results}
\subsection{Experimental Results on Imbalanced Dataset}
In Eq.~\ref{eq:balance}, we assume that the dataset is balanced. This section will further analyze this condition and show the robustness of our proposed method. We conduct the experiments under the setting of Tiny-ImageNet IPC 20. For each class of Tiny-ImageNet, we randomly select 60\%$\sim$100\% and 40\%$\sim$100\% of the original daraset to build the imbalanced dataset. The results are shown in the Table~\ref{tab:R2}. Even if imbalanced models are used, our method still demonstrates some resistance to this issue compared with the baseline as shown in Table~\ref{tab:R2}. 
This can be attributed to more informative statistics from diversified teachers for model-to-data synthesis. 
Moreover, the performance could be further improved by considering the class-wise statistics.

\subsection{Experimental Results on Larger Network}
To further demonstrate the effectiveness of our proposed method for larger networks, we conduct the experiments on ResNet50. Here, ResNet50 is utilized as the model pool base architecture, and the downstream network. The results are shown in the Table~\ref{tab:R3}~(left). Our method achieves superior performance for larger models.

\subsection{Experimental Results on Other Task}
Current DD field mainly focuses on classification, with detection and other tasks research still blank.
Here, we simply apply our method on the detection task to show the generalizability of our proposed method across different tasks. We conduct the experiments on Pascal VOC, and the architecture of the detector is Faster RCNN.
Table~\ref{tab:R3}~(right) shows the preliminary results of our method and baseline adapted to object detection on Pascal VOC, demonstrating the potential of our method in wider applications. 

\subsection{Experimental Results on Small Dataset}
Our proposed method, Teddy, is designed mainly for large-scale datasets as we do several Taylor-approximation strategies to improve the efficiency of the original solution to the DD definition, which may cause information loss. However, it still shows quite competitive performance compared with baselines SRe$^2$L and DM on small datasets. Here, we adopt CIFAR-10 and CIFAR-100 for experiments. Specifically, we drop the soft label adopted in both SRe$^2$L and our method for fair comparison. Also, we only use the DSA strategy for data augmentation. The evaluation results are shown in Table~\ref{tab:wosoftlabel}. It shows that our weak-teacher strategy is a more favorable surrogate of the original ones than fully-converged models~(SRe$^2$L) and random initialization without training~(DM).

\section{More Implementation Details}
\subsection{Model Pool Generation}
In our experiments, we have two strategies to generate the model pool: prior-generation and post-generation.
The base network architecture is ResNet18.
For prior-generation, we follow the official torchvision code to train ResNet18 for ImageNet-1K and modified ResNet18 for Tiny-ImageNet. We adopt different stage teacher models for different IPC settings. The size of the model pool is 9 for ImageNet-1K and 8 for Tiny-ImageNet. For more details, please refer to Table~\ref{tab:modelpoold}.

As for post-generation, we utilize the pre-trained ResNet18 provided by the official torchvision for ImageNet-1K and well-trained modified ResNet18 for Tiny-ImageNet. We use DepGraph to perform structural pruning with the random standard, and finetune the pruned models for very limited epochs, \textit{e.g.}, 0-2 epochs, with learning rate from 0.1 to 0.001. It is worth noting that during finetuning, we simulate the accuracy curve of teacher models in trajectory-based model pool. The GFLOPs of the target pruned model is 1.2G, and the number of parameters is 7.72M. The size of model pool is 10 for ImageNet-1K and 9 for Tiny-ImageNet. For more details, please refer to Table~\ref{tab:modelpoold}.

\begin{table}[!t]
    \centering
    \resizebox{\linewidth}{!}{
    \begin{tabular}{c c c c c c c c c c c c}
        \toprule
        \multirow{2}{*}{\bf Method} & & \multicolumn{3}{c}{\bf CIFAR-10} & &\multicolumn{3}{c}{\bf CIFAR-100} \\
    \cmidrule{3-5} \cmidrule{7-9}
    & & 1 & 10 & 50 && 1 & 10 & 50 \\
            % Real / Noise & CIFAR10-IPC 1 & 10 & 50 & CIFAR100-IPC 1 & 10 & 50  \\
        \midrule
        \bf DM~(noise) && 25.6 $\pm$ 0.2 & 49.8 $\pm$ 0.3 & 59.5 $\pm$ 0.1 && 11.4 $\pm$ 0.2 & 29.6 $\pm$ 0.3 & 36.9 $\pm$ 0.1 \\
         \bf DM~(real) && 26.0 $\pm$ 0.8 & 48.9 $\pm$ 0.6 & 63.0 $\pm$ 0.4 && 11.4 $\pm$ 0.3 & 29.7 $\pm$ 0.3 & 43.6 $\pm$ 0.4 \\
        \bf SRe$^2$L w/o soft label~(noise) && 24.9 $\pm$ 0.2 & 35.8 $\pm$ 0.4 & 37.0 $\pm$ 0.2 && 10.8 $\pm$ 0.2 & 18.4 $\pm$ 0.1 & 25.6 $\pm$ 0.1 \\
        \bf SRe$^2$L w/o soft label~(real) && 27.1 $\pm$ 0.4  & 44.7 $\pm$ 0.2 & 55.7 $\pm$ 0.2  && 10.7 $\pm$ 0.3 & 24.6 $\pm$ 0.1 & 39.2 $\pm$ 0.1 \\
        \midrule
        \bf Ours w/o soft label~(noise) && 29.2 $\pm$ 0.3 & 51.0 $\pm$ 0.1 & 63.6 $\pm$ 0.1 && 12.5 $\pm$ 0.2 & 30.6 $\pm$ 0.1 & 44.4 $\pm$ 0.2\\
        \bf Ours w/o soft label~(real) && \bf 30.1 $\pm$ 0.4  & \bf 53.0 $\pm$ 0.3 & \bf 66.1 $\pm$ 0.1  && \bf 13.5 $\pm$ 0.2 & \bf 33.4 $\pm$ 0.1  & \bf 49.4 $\pm$ 0.4 \\  
        
        \bottomrule
    \end{tabular}
    }
    \caption{Results on small datasets CIFAR-10 and CIFAR-100. For a fair comparison, we drop the soft label adopted in both SRe$^2$L and our proposed method and only use DSA as the data augmentation strategy.}
    \label{tab:wosoftlabel}
\end{table}

\subsection{Data Generation}
For ImageNet-1K, we randomly select 3 models from the prior-generated model pool or 4 models from the post-generated model pool to generate the $i^{th}$ image for all classes. As for Tiny-ImageNet, we randomly select 3 models from prior-generated model pool or 5 models from the post-generated model pool. The batch size is 100, and the number of synthetic data generation iterations is 6000. For more details, please refer to Table~\ref{tab:trainingd}.

\begin{table}
    \centering
    \resizebox{\linewidth}{!}{
    \begin{tabular}{c l c c c c c}
        \toprule
        & \bf Hypeparameter && \bf Prior-Generation & \bf Post-Generation\\
        \midrule
        \multirow{10}{*}{\rotatebox[origin=c]{90}{\bf Tiny-ImageNet}}
        & Optimizer && SGD & SGD\\
        & Base Learning Rate && 0.2 & 0.1-0.001\\
        & Learning Rate Scheduler && Cosine & Step \\
        & Weight Decay && 1e-4 & 1e-4\\
        & Learning Rate Step Size && - & 0-2 \\
        & Momentum && - & 0.9 \\
        & Batch Size && 256 & 32\\
        & Model Pool Size && 8 & 9\\
        & Training / Finetuning Epochs && 46, 50 & 0-2\\
        & Stage of Teachers && 11-46, with step 5; 16-50, with step 5 & -\\
        \midrule
        \multirow{10}{*}{\rotatebox[origin=c]{90}{\bf ImageNet-1K}}
        & Optimizer && SGD & SGD\\
        & Base Learning Rate && 0.1 & 0.1-0.001\\
        & Learning Rate Scheduler && Step & Step \\
        & Weight Decay && 1e-4 & 1e-4\\
        & Learning Rate Step Size && 30 & 0-2 \\
        & Momentum && 0.9 & 0.9 \\
        & Batch Size && 32 & 32\\
        & Model Pool Size && 9 & 10\\
        & Training / Finetuning Epochs && 41, 61, 71 & 0-2\\
        & Stage of Teachers && 1-41, with step 5; 21-61, with step 5; 31-71, with step 5 & -\\
        \bottomrule
    \end{tabular}
    }
    \caption{The hyper-parameters for model pool generation part.}
    \label{tab:modelpoold}
\end{table}

\begin{table}
    \centering
    \resizebox{\linewidth}{!}{
    \begin{tabular}{c l c c c c c}
        \toprule
        & \bf Hypeparameter && \bf Prior-Generation & \bf Post-Generation \\
        \midrule
        \multirow{10}{*}{\rotatebox[origin=c]{90}{\bf Tiny-ImageNet}}
        & Optimizer && Adam & Adam\\
        & Base Learning Rate && 0.1 & 0.1\\
        & Learning Rate Scheduler && Cosine & Cosine \\
        & Weight of BN Loss && 1.0 & 1.0\\
        & Weight Decay && 1e-4 & 1e-4\\
        & Momentum && 0.5, 0.9 & 0.5, 0.9 \\
        & Batch Size && 100 & 100\\
        & Model Pool Size && 8 & 9\\
        & Generation Iterations && 4000 & 4000 \\
        & Number of Teachers Ensemble && 3 for data generation at one time & 5 for data generation at one time \\
        \midrule
        \multirow{10}{*}{\rotatebox[origin=c]{90}{\bf ImageNet-1K}}
        & Optimizer && Adam & Adam\\
        & Base Learning Rate && 0.25 & 0.25\\
        & Learning Rate Scheduler && Cosine & Cosine \\
        & Weight of BN Loss && 0.01 & 0.01\\
        & Weight Decay && 1e-4 & 1e-4\\
        & Momentum && 0.5, 0.9 & 0.5, 0.9 \\
        & Batch Size && 100 & 100\\
        & Model Pool Size && 9 & 10\\
        & Generation Iterations && 6000 & 6000 \\
        & Number of Teachers Ensemble && 3 for data generation at one time & 4 for data generation at one time\\
        \bottomrule
    \end{tabular}
    }
    \caption{The hyper-parameters for data generation part.}
    \label{tab:trainingd}
\end{table}
\subsection{Validation}
For the validation part, the generated data is first augmented using the CutMix strategy, and then the soft labels of the augmented data are generated by the model pool. The number of validation epochs for Tiny-ImageNet is 100, and for ImageNet-1K is 300. The training batch size is 256 for Tiny-ImageNet and 1024 for ImageNet-1K. The optimizer we adopted is AdamW, and the optimizer momentum is 0.9 and 0.999, respectively. The base learning rate is 0.001, the weight decay is 0.01, and the learning rate scheduler is Cosine.

\section{More Visualization Results}
Here, we provide additional visualizations of the synthetic dataset generated by our method, including different settings on Tiny-ImageNet~(Fig.~\ref{fig:vis2}) and ImageNet-1K~(Fig.~\ref{fig:vis3}). From a visualization perspective, our method showcases a broader range of perspectives and diversity. Here, for instance, in the whistle category, it can be observed that the focal points of the generated model pool by different methods are distinct. Post-generation focuses on the object itself, while prior-generation also pays attention to the surrounding environment, \textit{e.g.}, the animal blowing the whistle.

\begin{figure}[t]
  \centering
   \includegraphics[width=\linewidth]{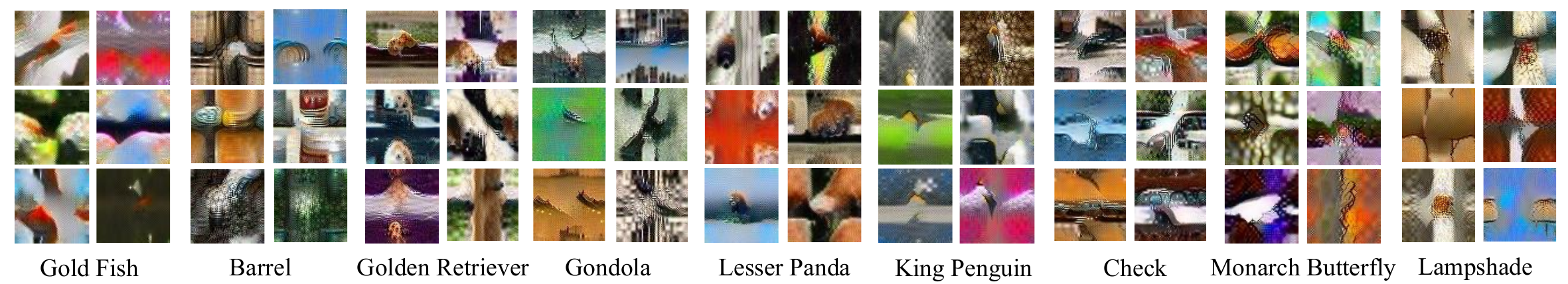}
   \caption{Visualization results of our proposed method for Tiny-ImageNet IPC 50.}
   \label{fig:vis2}
\end{figure}

\begin{figure}[t]
  \centering
   \includegraphics[width=\linewidth]{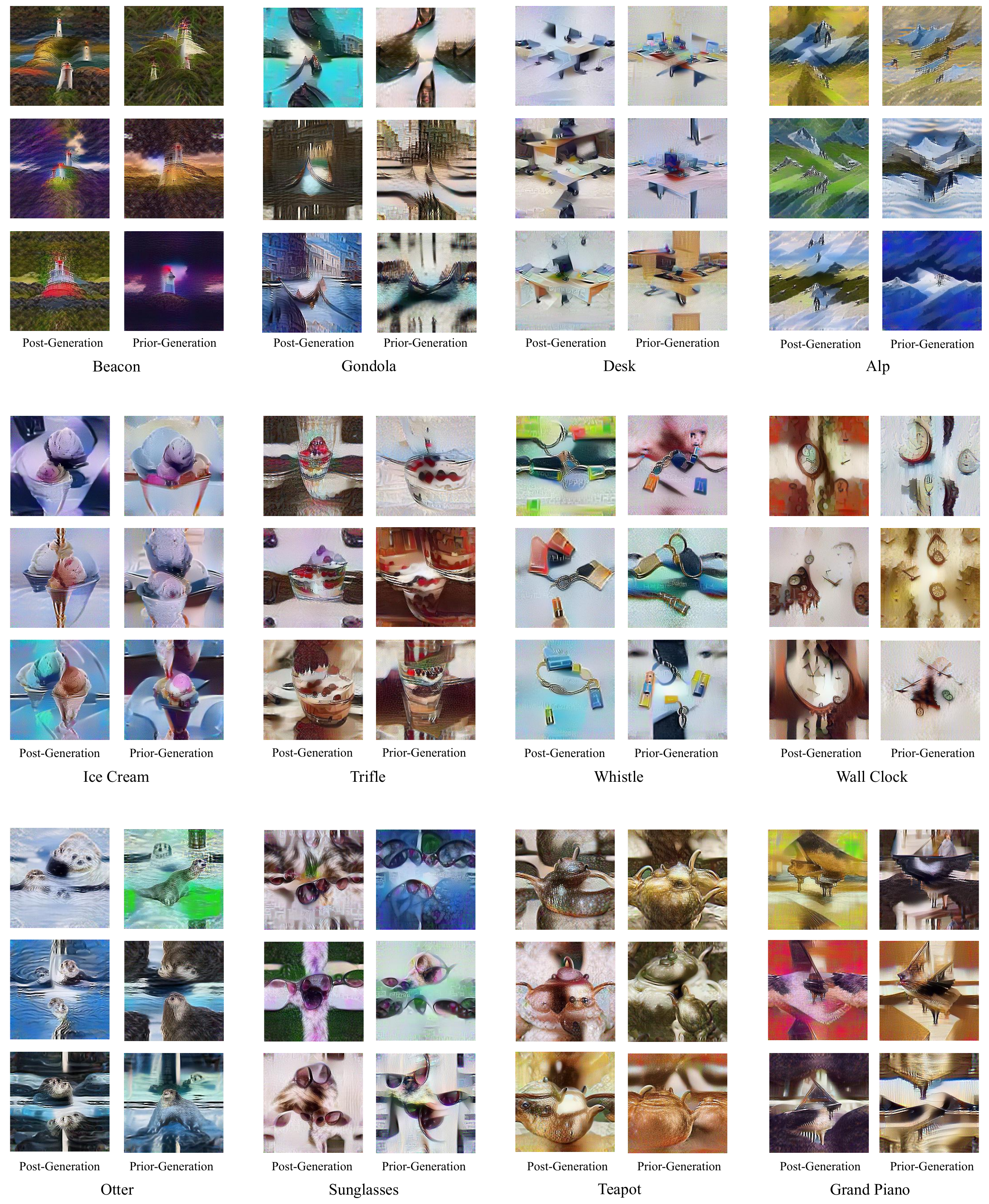}
   \caption{Visualization results of our proposed method for ImageNet-1K IPC 50.}
   \label{fig:vis3}
\end{figure}

\end{document}